\documentclass{article}

\usepackage[a4paper, total={6in, 9in}]{geometry}

\usepackage{bm}
\usepackage{url} 
\usepackage{dsfont}
\usepackage{xcolor}
\usepackage{amsthm}
\usepackage{amsmath}
\usepackage{amssymb}
\usepackage{multicol}
\usepackage{mathtools}

\usepackage{algorithm}
\usepackage{algorithmic}

\usepackage[round]{natbib}

\usepackage[toc, page, titletoc]{appendix}

\newtheorem{lemma}{Lemma}
\newtheorem{remark}{Remark}
\newtheorem{theorem}{Theorem}

\newtheorem{assumption}{Assumption}
\newenvironment{hproof}{\proof}{\endproof}
\newenvironment{tproof}{\proof}{\endproof}

\usepackage{thm-restate}

\DeclareMathOperator{\Tr}{Tr}
\DeclareMathOperator*{\argmax}{arg\,max}

\DeclarePairedDelimiter{\abs}{\lvert}{\rvert}
\DeclarePairedDelimiter{\norm}{\lVert}{\rVert}
\DeclarePairedDelimiter{\ceil}{\lceil}{\rceil}

\title{Delayed Feedback in Generalised Linear Bandits Revisited}

\author{Benjamin Howson \and Ciara Pike-Burke \and Sarah Filippi}
\date{
    Department of Mathematics, Imperial College London\\[2ex]
    \today
}

\allowdisplaybreaks
\begin{document}
\maketitle
\begin{abstract}
The stochastic generalised linear bandit is a well-understood model for sequential decision-making problems, with many algorithms achieving near-optimal regret guarantees under immediate feedback. However, the stringent requirement for immediate rewards is unmet in many real-world applications where the reward is almost always delayed. We study the phenomenon of delayed rewards in generalised linear bandits in a theoretical manner. We show that a natural adaptation of an optimistic algorithm to the delayed feedback achieves a regret bound where the penalty for the delays is independent of the horizon. This result significantly improves upon existing work, where the best known regret bound has the delay penalty increasing with the horizon. We verify our theoretical results through experiments on simulated data. 
\end{abstract}

\section{Introduction}\label{section: introduction}
Recently, bandit algorithms have found application in areas from dynamic pricing and healthcare to finance and recommender systems with great success \citep{Misra2019, Durand2018a, Shen2015, Spotify2018}. There are many formulations of bandit problems. One of these is the stochastic generalised linear bandit, which captures a wide class of problems, such as when the rewards are counts, binary values or can take any real-valued number. The generalised linear bandit problem proceeds in rounds, where in each round, a learner must choose from a set of possible actions. After selecting an action, the learner receives feedback from the environment in the form of a reward which stochastically depends on the inner product of the action and some unknown parameter vector. The goal of the learner is to maximise their expected cumulative reward.\\

There are many provably efficient algorithms for the generalised linear bandit \citep{Filippi2010, Abbasi-Yadkori2011, Li2017, Faury20a}. Unfortunately, these existing algorithms require immediate feedback from the environment. This strict requirement for immediate rewards often goes unmet in practice. For example, in many recommender systems, the user must provide feedback to the learner while operating on a very different time scales; e.g. the learner can make thousands of recommendations per second, whereas a user may take several minutes to a couple of days to respond to the recommendation, if at all \citep{Chapelle2014}. Alternatively, practitioners might want to optimise for a longer-term measure of success \citep{Lyft2021}, in which case the reward is not observable or even defined immediately. Delayed feedback also arises in clinical trials due to the time-consuming task of obtaining medical feedback and because patients do not respond to their prescribed treatment immediately.\\

In all the above settings, the reward for any given action returns at an unknown time in the future. Meanwhile, the learner must continue operating in the environment without feedback from many of their past choices. A natural model for this phenomenon is to introduce a random delay between taking action and receiving the reward. However, the delays pose significant theoretical challenges because standard tools for analysing bandit algorithms rely on utilising immediate feedback to reduce the uncertainty in the learner's estimation. Under delayed feedback, it is unclear how long the learner will have to wait before they gain information about the quality of an action, which hinders their future decision-making abilities.\\ 

These challenges have led to the development of algorithms specifically for delayed feedback in generalised linear bandits. However, to the best of our knowledge, these existing algorithms require \textit{a-priori} knowledge of the expected delay (along with other quantities), strong assumptions on the delay distribution, restrictive assumptions on the action sets, or any combination thereof. Moreover, the best regret bound achievable by these algorithms is $\widetilde{\mathcal{O}}(\sqrt{dT}\sqrt{d + \mathbb{E}[\tau]})$ where $T$ is the number of rounds, $d$ is the dimension of the unknown parameter, and $\mathbb{E}[\tau]$ is the expected delay. This result suggests that as the horizon increases, the impact of the delayed feedback will increase. This result is counter intuitive and starkly differs from the results in the $K$-armed bandit setting where the impact of the delay is independent of the horizon \citep{Joulani2013}. In this paper, we prove that a simple algorithm based on optimism can achieve a regret bound of $\widetilde{\mathcal{O}}(d\sqrt{T} + d^{3/2}\mathbb{E}[\tau])$. This improves the penalty for delayed feedback from $\sqrt{dT \mathbb{E}[\tau]}$ in prior work to $d^{3/2} \mathbb{E}[\tau]$, separating the delay penalty  from the horizon.\\

\subsection{Related Work}
The multi-armed bandit literature covers stochastically delayed feedback extensively. Both \citet{Joulani2013} and \citet{Mandel2015} propose queue-based approaches to adapt existing $K$-armed bandit algorithms to delayed feedback, each proving that the regret bound of the chosen algorithm only increases by an additive factor involving the expected delay. \cite{Pike-Burke2018} study another version of delayed feedback, where the rewards from various rounds are not only delayed but also aggregated. \cite{Vernade2017} consider the setting of delayed conversions, where actions associated with long delays can have censored feedback.\\ 

Comparatively, fewer theoretical results quantify the impact of delays beyond $K$-armed bandits. \citet{Vernade2020} consider a Bernoulli bandit with censored rewards whose expected value is linear in some unknown parameter vector. Combining Bernoulli rewards with delays makes it impossible to distinguish between a reward of zero and a delayed reward. Thus, the challenges they face are different to ours. Nevertheless, they deal with the delays by inflating the exploration bonus and handle the censoring by introducing a windowing parameter that sets rewards taking too long to return equal to zero. \citet{Dudik2011} develop a policy elimination algorithm capable of handling contextual information and prove a regret bound of the form $\widetilde{\mathcal{O}}(\sqrt{KT} + \sqrt{K}\tau)$, where $K$ is the number of actions and $\tau$ is a constant delay between playing an action and observing the corresponding reward. However, they remark that their algorithm is challenging to implement, requires perfect knowledge of the distribution of the contextual information, and needs a-priori knowledge of the constant delay. \\

\citet{Zhou2019} and follow-up work by \citet{Blanchet2020}, that analyses that same algorithm, consider learning in the same setting as us. They propose an optimistic algorithm that inflates the exploration bonus by the square root of the number of missing feedbacks. They do this to account for the uncertainty arising from the missing rewards. Combining this bonus with an elegant argument allows them to use standard theoretical tools to handle the leading-order terms, namely the elliptical potential lemma. This lemma has found applications in the analysis of many linear bandit algorithms and is provably tight \citep{Carpentier2020}. However, due to the delay-dependent bonuses, their arguments lead to a multiplicative increase in the regret of the form $\widetilde{\mathcal{O}}(d\sqrt{T} + \sqrt{dT(\mathbb{E}[\tau] + M_{\tau})})$, where $M_{\tau}$ is a known non-negative delay-dependent constant beyond which the delays have tails that are as heavy as (or lighter than) the exponential distribution. However, this algorithm requires prior knowledge of the expected delay and $M_{\tau}$ (along with other quantities). This theoretical result suggests that the impact of the delayed feedback increases with the horizon, which does not align with the intuition that the delays become irrelevant once the learner has observed enough feedback to obtain a "good" estimate of the underlying expected reward function.\\

\subsection{Contributions}
In this paper, we present a natural approach based on optimism that does not require any prior knowledge of the delays and achieves regret bound of $\widetilde{\mathcal{O}}(d\sqrt{T} + d^{3/2}\mathbb{E}[\tau]\,)$, up to problem-specific constants. This result significantly improves upon the best-known theoretical results for generalized linear bandits with delayed feedback, whose regret bound is $\widetilde{\mathcal{O}}(d\sqrt{T} +\sqrt{dT \mathbb{E}[\tau]}\,)$. Further, our results align with what is seen in the $K$-armed bandit setting, where the delays only impact the worst-case performance by an additive penalty involving the expected delay \citep{Joulani2013}, and not the horizon $T$.\\

In contrast to prior work, we forgo the period of forced exploration which is present in many generalised linear bandit algorithms \citep{Filippi2010, Li2017}. Our algorithm is optimistic and constructs optimistic estimates using only \emph{observations that have returned}. To do this, we develop delay-adapted confidence sets and prove that these are valid. Although this algorithm is natural, proving regret bounds for it is somewhat involved. In particular, the presence of delayed feedback obscures how selecting a sub-optimal action in round $t$ will improve the estimation in future rounds. To overcome these issues we provide a novel analysis centered around an elliptical potential lemma for delayed feedback, which may be of independent interest for bandit algorithms with complex feedback structures. We show that this technique leads to the stated regret bound. We also validate our theoretical findings experimentally in some simulated environments.\\

\section{Problem Formulation}\label{section: problem formulation}
The stochastic generalised linear bandit problem considers learning in an environment where the expected reward is a known function of the dot product between the action and the unknown parameter vector. Letting $X_{t}\in\mathcal{A}_{t}\subset \mathbb{R}^d$ and $Y_{t}\in \mathbb{R}$ be the action and reward associated with the $t$-th round, we assume that the conditional distribution of the reward given the action belongs to the exponential family:
\begin{equation}\label{equation: likelihood}
f\left(Y_{t}\,\vert\, X_{t},\, \theta^{*} \right) \propto \exp\left(\frac{Y_{t} X_{t}^{T} \theta^{*} - b\left(X_{t}^{T} \theta^{*}\right)}{a\left(\phi\right)}\right) 
\end{equation}
where $\theta^{*}\in\mathbb{R}^{d}$ is an unknown parameter vector; $a$ and $b$ are known distribution specific functions; and $\phi$ is a known constant that is often referred to as the \textit{dispersion parameter}. For distributions belonging to the exponential family, one can verify that: 
\begin{align*}
    \mathbb{E}\left[Y_{t}\,\vert\, X_{t}\right] = \mu\left(X_{t}^{T}\theta^{*}\right) = \dot{b}\left(X_{t}^{T} \theta^{*}\right).
\end{align*}
Here, $\mu(\cdot)$ is a strictly increasing \textit{link function} that relates the inner product of the action vector and the unknown parameter to the expected reward. For example, if the rewards are normally distributed, $\mu(z) = z$ and we recover the standard linear model. If the rewards are Bernoulli, then $\mu(z) = 1/(1 + \exp(-z))$ and we have a logistic regression model.\\

In the stochastic setting, the learner selects an action $X_t \in \mathcal{A}_{t}\subset \mathbb{R}^d$ and receives noisy observations of the unknown expected reward function of the form $Y_{t} \sim f\left(Y_{t}\,\vert X_{t},\, \theta^{*}\right)$ where
$$
\eta_{t} \coloneqq  Y_{t} - \mu\left(X_{t}^{T}\theta^{*}\right) 
$$
is the noise and is zero-mean conditional on past decisions and rewards. Section \ref{section: problem formulation - assumptions} formally states the assumptions we make on the link function and the noise.\\ 

The ultimate goal of the learner in the generalised linear bandit setting is to minimise the regret. Intuitively, this compares the expected reward of the action selected by the learner to the action with the highest expected reward. Mathematically, we define the regret of an algorithm in the generalised linear bandit setting as follows: 
\begin{equation}\label{equation: regret}
    \hat{R}_{T} = \sum_{t = 1}^{T} \mu\left(\langle X_{t}^{*}, \theta^{*}\,\rangle\right) - \mu\left(\langle X_{t}, \theta^{*}\,\rangle\right)\coloneqq \sum_{t = 1}^{T} \hat{r}_{t}
\end{equation}
where $X_{t}^{*} = \argmax_{x\in\mathcal{A}_{t}}\{\mu\left(\langle x, \,\theta^*\rangle\right)\}$ is the action in the decision set $\mathcal{A}_{t}$ maximising the expected reward in the $t$-th round. 

\subsection{Delayed Feedback Learning Setting}\label{section: problem formulation - learning setting}
Let $\tau_{t}\in[0, \infty)$ denote the random delay associated with the decision made in the $t$-th round. Then, the sequential decision-making procedure for generalized linear bandits under stochastically delayed feedback is as follows. For $t \in \{1, \cdots, T\}$:
\begin{itemize}
\item [1.] The learner receives a decision set: $\mathcal{A}_{t} \subset \mathbb{R}^{d}$ where $\norm{x}_2 \leq 1$ for all $x \in \mathcal{A}_{t}$.
\item [2.] The learner selects a $d$-dimensional feature vector from the decision set: $X_{t} \in \mathcal{A}_{t}$.
\item [3.] Unbeknownst to the learner, the environment generates a random delay, a random reward and then schedules an observation time:
\begin{itemize}
\item [3a.] The random reward has the form: $Y_{t} = \mu\left(X_{t}^{T}\theta^{*}\right) + \eta_{t}\,.$
\item [3b.] The environment schedules the observation time of the reward: $\ceil{t + \tau_{t}}$ where $\tau_{t} \sim f_{\tau}\left(\cdot\right)\,.$
\end{itemize}
\item [5.] The learner receives delayed rewards from its previous actions: $\{(s, Y_s): t - 1< s + \tau_s \leq t\}$.
\end{itemize}

From the above decision-making procedure, it is clear that the learner only has access to the rewards of the actions whose observation times are less than or equal to $t - 1$ when making decisions in round $t$. Therefore, $Y_s$ is observable to the learner in the rounds where $s + \tau_s \leq t - 1$. Otherwise, it is missing. To that end, we define the $\sigma$-algebra generated by the set of observable information at the start of the $t$-th round as: 
$$
\mathcal{F}_{t - 1} = \sigma \left(\left\{\left(X_s, C_s^{t - 1}, Y_s C_s^{t - 1}\,\right): s \leq t - 1\right\}\,\cup \mathcal{A}_{t}\right)
$$
where 
$$
C_{s}^{t} = \mathds{1}\left\{s + \tau_{s} \leq t\right\}
$$
indicates whether the reward associated with the $s$-th round is observable at the end of the $t$-th round. Naturally, $C_s^t$ is observable at the end of each round, as the learner can easily check which actions have and have not received feedback; this is standard in most works on delays in the bandit literature \citep{Dudik2011, Joulani2013, Mandel2015, Zhou2019, Blanchet2020}. 
Notably, $C_s^t$ is $\mathcal{F}_{t}$-measurable, meaning the learner only has access to the indicators and the rewards observed at the end of rounds $1, \cdots, t - 1$ when making decisions in round $t$.

\subsection{Assumptions}\label{section: problem formulation - assumptions}
We make the following assumptions on the noise and the link function. These are standard in the literature on linear and generalised linear bandits \citep{Filippi2010, Abbasi-Yadkori2011, Li2017}.

\begin{assumption}[Subgaussian Noise]\label{assumption: subgaussian reward}
Let $R \geq 0$ and $\abs{\eta_{t}} \leq R$ almost surely. Then, the moment generating function of the noise distribution conditional on the observed information must satisfy the following inequality:
\begin{equation*}
     \mathbb{E}\left[\exp\left(\gamma\,\eta_{t}\right)\vert \mathcal{F}_{t - 1}\right] \leq \exp\left(\frac{1}{2}\gamma^{2}R^{2}\right)
\end{equation*}
for all $\gamma \in \mathbb{R}$.
\end{assumption}

\begin{assumption}[Link Function]\label{assumption: link function}
The link function $\mu: \mathbb{R} \rightarrow \mathbb{R}$ is known a-priori and is twice differentiable with first and second derivatives bounded by $L_{\mu}$ and $M_{\mu}$, respectively. Further, 
$$
\kappa \coloneqq \inf\left\{\dot{\mu}\left(\langle x, \theta \,\rangle\right): \left(x, \theta\right)\in \mathcal{A}\times \Theta \right\} > 0.
$$
where $\Theta$ is the set of all possible parameter vectors.
\end{assumption}

Assumption \ref{assumption: subgaussian reward} implies that the noise distribution has light tails. Assumption \ref{assumption: link function} implies that the link function is $L_{\mu}$-Lipschitz. One can interpret the condition on $\kappa$ as guaranteeing that it is possible to distinguish between two actions whose expected rewards are arbitrarily close to one another. Indeed, $R$, $L_{\mu}$ and $\kappa$ all feature in the theoretical analysis and regret bounds.\\

It will also be necessary for the delays to satisfy some assumptions (see Section \ref{section: delayed ofu-glm - regret bounds}). In particular, we assume the following holds.
\begin{assumption}[Subexponential Delays]\label{assumption: subexponential delays}
The delays are non-negative, independent and identically distributed $(v, b)$-subexponential random variables. That is, their moment generating function satisfies the following inequality:
\begin{equation*}
    \mathbb{E}\left[\exp\left(\gamma \left(\tau_{t} - \mathbb{E}\left[\tau_{t}\right]\right)\right)\right] \leq \exp\left(\frac{1}{2}v^2 \gamma^2\right)
\end{equation*}
for some  non-negative $v$ and $b$, and all $\abs{\gamma} \leq 1/b$.
\end{assumption}
The class of distributions with subexponential tail behaviour is broad enough to include many heavy-tailed distributions, such as the $\chi^{2}$ and Exponential distributions. Importantly, Assumption \ref{assumption: subexponential delays} aligns with the empirical evidence suggesting that delays have exponential-like tails in practice \citep{Chapelle2014}. However, other tail bounds on the delays can be used if they exist. Furthermore, it is possible to relax this assumption to only requiring that the delays have a finite (unknown) expected value by considering the expected regret, a weaker theoretical guarantee. 

\subsection{Notation}
Throughout, $\norm{x}_{p}$ denotes the $p$-norm of an arbitrary vector $x\in\mathbb{R}^{d}$. For $A, B \in \mathbb{R}^{d\times d}$, we denote $\norm{x}_{A} = \sqrt{x^T A x}$ and adopt the following notation for positive (semi)-definite matrices: 
\begin{itemize}
    \item $A \succeq 0$ (positive semi-definite) $\iff \norm{x}_{A}^{2} \geq 0$ for all $x\in\mathbb{R}^{d}$.
    \item $A \succeq B$ $\iff \norm{x}_{A}^2 \geq \norm{x}_{B}^{2}$ for all $x\in\mathbb{R}^{d}$.
\end{itemize}
Additionally, $\lambda_{i}(A)$ and $\sigma_{i}(A)$ denote the $i$-th largest eigenvalue and the $i$-th largest singular value of matrix $A$, respectively. Finally, we denote the first and second derivatives of a real-valued function $f$ by $\dot{f}$ and $\Ddot{f}$, respectively. 

\section{Delayed OFU-GLM}\label{section: algorithm}
In this section, we describe a provably efficient algorithm for generalised linear bandits with stochastic delays. We base our approach on the optimistic principle and show that delays only cause an additive increase in the regret bound. This is in contrast to the multiplicative effect seen in existing work \citep{Blanchet2020}.\\

Due to the delays, it is necessary to introduce some additional notation that discriminates between rounds whose feedback has or has not been observed. Denote the number of missing rewards at the end of the $t$-th round by:
\begin{equation*}
    G_{t} = \sum_{s = 1}^{t}\mathds{1}\left\{s + \tau_{s} > t\right\}\;.
\end{equation*}
Further, we define the total, observed and missing design matrices as
\begin{align}
    \bar{V}_{t} &= \lambda I + \sum_{s = 1}^{t} X_{s}X_{s}^{T} \label{equation: total design matrix}\\
    \bar{W}_{t}  &= \lambda I + \sum_{s = 1}^{t}  \mathds{1}\{s + \tau_{s} \leq t\}X_{s}X_{s}^{T}\label{equation: observed design matrix}\\
    Z_{t} &= \sum_{s = 1}^{t}  \mathds{1}\{s + \tau_{s} > t\} X_{s}X_{s}^{T}\label{equation: missing design matrix}\,,
\end{align}
respectively. Here, $\lambda > 0$ is a regularisation parameter. Briefly, $\bar{V}_{t}$ is the total design matrix and contains information relating to all past choices. Whereas $\bar{W}_{t}$ and $Z_{t}$ include information about actions with and without observed rewards, respectively. It is easy to see that the total, observed and missing design matrices must satisfy the following relationship:
\begin{equation}\label{equation: matrix relationship}
    \bar{V}_{t} = \bar{W}_{t} + Z_{t}\;.
\end{equation}
Thus, when there are no delays, the total and observed design matrices are equivalent to each other, and the missing design matrix is full of zeros.

\subsection{Estimation Procedure}
As is standard when fitting generalised linear models, we use maximum likelihood estimation to estimate the unknown parameter of the environment. However, we make several adjustments to the estimator to account for delayed feedback.\\

First note that not all actions played will have received feedback. To mitigate this issue, we ignore the actions with missing feedback in our estimation procedure. Secondly, many existing algorithms for generalised linear bandits use a phase of pure exploration \citep{Filippi2010, Li2017}. This exploration phase lasts until the observed design matrix is of full rank, which ensures a unique maximiser of the likelihood function exists. Since we choose to ignore actions with missing feedback, the length of this exploration phase will depend on the delay distribution.\\

To avoid waiting for an exploration phase to pass, we introduce a penalisation term into the objective function, an idea that we borrow from the linear bandits where one can derive a closed-form penalised maximum likelihood estimator \citep{Abbasi-Yadkori2011, Chu2011}. In the generalised linear setting, this trick equates to penalising the log-likelihood function and has found use for logistic bandits under immediate feedback \citep{Jun2017, Faury20a}. From Equation \eqref{equation: likelihood} and the conditional independence of the rewards given past actions, one can write the penalised log-likelihood as follows:
\begin{equation}\label{equation: log likelihood}
    \mathcal{L}_{t}\left(\theta, \alpha\right) = \sum_{s = 1}^{t}C_{s}^{t}\,\log\left(f\left(Y_{t}\,\vert X_{t}\right)\right) 
    - \frac{\alpha}{2}\norm{\theta}_{2}^{2}
\end{equation}
Equation \eqref{equation: log likelihood} always has a unique maximiser due to the introduction of $\alpha > 0$, which means we can leverage new information from the very first round. One can easily verify that the maximiser is the solution of the following equation:
\begin{equation}\label{equation: mle}
    \left[\sum_{s = 1}^{t} C_{s}^{t}\left(Y_{s} - \mu\left(X_{t}^{T} \theta\right)\right) X_{s}\right] - \alpha\, a(\phi) \, \theta = 0 
\end{equation}
where $a(\phi)$ is a known function of the dispersion parameter of the reward distribution. We denote the solution of Equation \eqref{equation: mle} by $\hat{\theta}_t$. To implement the optimistic principle, we construct confidence sets around our estimators and prove that this set contains $\theta^{*}$ with high probability. 

\begin{restatable}{lemma}{confset}\label{lemma: ofu-glm confidence set}
Let $\lambda = \alpha \,a(\phi) /\kappa$ and assume that $\norm{\theta^{*}}_{2} \leq m_{1}$. Then, with probability at least $1 - \delta$, for all rounds $t\geq 0$:
$$
\norm{\hat{\theta}_{t} - \theta^{*}}_{\bar{W}_{t}} \leq  \sqrt{\lambda}m_{1} + \frac{R}{\kappa} \sqrt{2\log\left(\frac{\det\left(\bar{W}_{t}\right)^{1/2}}{\delta\,\lambda^{d/2}}\right)}
$$
\end{restatable}
\begin{hproof}
Firstly, we account for regularising the log-likelihood function, which we do in Lemmas \ref{lemma: mean value theorem} and \ref{lemma: separate bias contribution} of Appendix \ref{section: confidence sets}. These lemmas allow us to separate noise-related terms from those introduced by biasing our estimator with the regularisation term. Subsequently, we show that the noise-related terms satisfy the martingale property under the information structure created by the delays. This result allows us to apply existing results for self-normalising processes \citep{Pena2004, Abbasi-Yadkori2011}. See Appendix \ref{section: confidence sets} for a full proof.  
\end{hproof}
By Lemma \ref{lemma: ofu-glm confidence set}, defining the confidence sets as:
\begin{equation}\label{equation: confidence set}
    \mathcal{C}_{t} = \left\{\theta \in \mathbb{R}^{d}:\norm{\hat{\theta}_{t} - \theta }_{\bar{W}_{t}} \leq \sqrt{\beta_{t}}\right\}
\end{equation}
with 
\begin{equation}
    \sqrt{\beta_{t}} = \sqrt{\lambda} m_{1} + \frac{R}{\kappa} \sqrt{2\log\left(\frac{\det\left(\bar{W}_{t}\right)^{1/2}}{\delta\,\lambda^{d/2}}\right)}\label{equation: confidence width}
\end{equation}
guarantees that $\mathbb{P}(\exists\,t \geq 0: \theta^{*}\not\in \mathcal{C}_{t}) \leq 1 - \delta$. 

\subsection{Delayed OFU for Generalised Linear Bandits}
Algorithm \ref{algorithm: delayed ofu-glm} presents the pseudo-code for our algorithm, Delayed OFU-GLM. It requires several input parameters that we briefly discuss below. 

\begin{algorithm}\caption{Delayed OFU-GLM}\label{algorithm: delayed ofu-glm}

\begin{algorithmic}
\STATE \textbf{Input:} model parameters $d$, $a(\phi)$, $m_{1}$, $\kappa$, and tuning parameters $\alpha > 0$ and $\delta \in (0, 1)$.

\STATE \textbf{Initialise:} $\hat{\theta}_{0} = \vec{0}$, $\lambda = \frac{\alpha\,a(\phi)}{\kappa}$ and $\bar{W}_{0} = \lambda I$
    \FOR{$t = 1$ {\bfseries to} $T$}
        \STATE Play $X_{t}$ where: 
        $$
        (X_{t}, \tilde{\theta}_{t}) = \argmax_{(x, \theta)\in\mathcal{A}_{t}\times \mathcal{C}_{t - 1}} \mu\left(x^{T} \theta\right) 
        $$
        \STATE Receive the (possible empty) set of delayed rewards.
        \STATE Update $\bar{W}_{t}$, $\hat{\theta}_{t}$ and $\beta_{t}$ via Equations \eqref{equation: observed design matrix}, \eqref{equation: mle} and \eqref{equation: confidence width}. 
    \ENDFOR
\end{algorithmic}
\end{algorithm}

Firstly, algorithm requires knowledge of $a(\phi)$, a known function of the dispersion parameter of the reward distribution. For Bernoulli and Poisson rewards, one can show that $a(\phi) = 1$. In the Gaussian case, this parameter is the variance of the reward distribution $a(\phi) = R^2$, which all optimistic algorithms require to define the confidence sets.\\

Secondly, $m_{1} \geq \norm{\theta^{*}}_{2}$ is an upper bound on the $\ell_{2}$-norm of the unknown parameter vector that features in many existing algorithms for the immediate feedback setting \citep{Abbasi-Yadkori2011, Jun2017, Faury20a}. Note that since \citep{Zhou2019} uses a period of explicit exploration, they do not need this hyperparameter. Instead, they require knowledge of the delay distribution to define the length of the exploration phase. \\

Finally, $\kappa$ quantifies the smallest possible rate of change in the expected reward function. For Linear bandits with Gaussian rewards, $\kappa = 1$. For other distributions, one can replace this quantity with a lower bound and our theoretical results will still hold. For Logistic bandits, one can utilise the fact that the first derivative of the link function is symmetric about zero and decreasing to show that: $\kappa \geq  m_2 \coloneqq \dot{\mu}(m_{1})$. For Poisson bandits, by the definition of the inner product, $\kappa \geq m_2 \coloneqq \exp(-m_{1})$. Indeed, many optimistic algorithms for generalised linear bandits require this hyperparameter, as it features in the definition of the confidence sets. Recent work removes the need to specify this hyperparameter for the logistic bandit \citep{Faury20a}.

\subsection{Regret Bounds for Delayed OFU-GLM}\label{section: delayed ofu-glm - regret bounds}
In this subsection, we state and prove a worst-case regret bound for our algorithm. Specifically, Algorithm~\ref{algorithm: delayed ofu-glm} only suffers an additive penalty caused by the delays under the assumptions outlined in Section \ref{section: problem formulation - assumptions}. 

\begin{theorem}\label{theorem: delayed ofu-glm}
Suppose $\norm{x}_{2} \leq 1$ for all $x \in \cup_{t = 1}^{\infty} \mathcal{A}_{t}$, and Assumptions \ref{assumption: subgaussian reward}, \ref{assumption: link function} and \ref{assumption: subexponential delays} hold. Then, with probability greater than $1 - 3\delta$, Delayed OFU-GLM with any regularisation parameter $\lambda = \alpha\,a(\phi)/\kappa \geq 1$ has pseudo-regret that satisfies:
\begin{equation*}
    \hat{R}_{T} \leq \tilde{\mathcal{O}}\left(\frac{d R L_{\mu}}{\kappa}\sqrt{T} +  \frac{d^{3/2} RL_{\mu}\left(\mathbb{E}\left[\tau\right] + \min\left\{v, b\right\}\right)}{\kappa}\right)
\end{equation*}
where $v$ and $b$ are the subexponential parameters of the delay distribution.
\end{theorem}

\begin{proof}
First, we bound the per-round pseudo-regret. In Algorithm \ref{algorithm: delayed ofu-glm},the action selected by the algorithm is optimistic with probability $1 - \delta$. Therefore,
\begin{align*}
    \hat{r}_{t}  &= \mu\left(\langle  \theta^{*},\, X_{t}^{*}\rangle\right) - \mu\left(\langle\theta^{*},\, X_{t}\rangle\right)\\
    &\leq L_{\mu}\left(\langle  \theta^{*},\, X_{t}^{*}\rangle - \langle\theta^{*},\, X_{t}\rangle\right) \tag{Assumption \ref{assumption: link function}}\\
    &\leq L_{\mu}\left(\langle \tilde{\theta}_{t} - \theta^{*}, \, X_{t}\rangle\right) \tag{Lemma \ref{lemma: ofu-glm confidence set}}
\end{align*}
where the final inequality holds with probability at least $1 - \delta$ across all rounds due to the definition of the confidence sets and the action-selection procedure in Algorithm \ref{algorithm: delayed ofu-glm}. Adding and subtracting the maximum likelihood estimator gives: 
\begin{align*}
    \hat{r}_{t}  &\leq L_{\mu}\left(\langle \tilde{\theta}_{t} - \hat{\theta}_{t - 1}, \, X_{t}\rangle + \langle \hat{\theta}_{t - 1} - \theta^{*}, \, X_{t}\rangle\right)\\
    &\leq L_{\mu}\norm{\tilde{\theta}_{t} - \hat{\theta}_{t - 1}}_{\bar{W}_{t - 1}}\norm{X_{t}}_{\bar{W}_{t - 1}^{-1}} + L_{\mu}\norm{\hat{\theta}_{t - 1} - \theta^{*}}_{\bar{W}_{t - 1}}\norm{X_{t}}_{\bar{W}_{t - 1}^{-1}} \tag{H\"{o}lder's}\\
    &\leq 2L_{\mu}\sqrt{\beta_{t - 1}}\norm{X_{t}}_{\bar{W}_{t - 1}^{-1}} \tag{Definition of $\mathcal{C}_{t - 1}$}\\
    &\leq 2L_{\mu}\sqrt{\beta_{T}}\norm{X_{t}}_{\bar{W}_{t - 1}^{-1}} \tag{$\beta_{1} \leq \beta_{2} \leq \cdots \leq \beta_{T}$}
\end{align*}
Therefore, we have that the pseudo-regret has the following upper bound:
\begin{equation}\label{equation: terms to bound}
    \hat{R}_{T}
    \leq 2L_{\mu}\sqrt{\beta_{T}}\sum_{t = 1}^{T}\norm{X_{t}}_{\bar{W}_{t - 1}^{-1}}
\end{equation}
with probability at least $1 - \delta$. Usually, an application of Cauchy-Schwarz and the elliptical potential lemma handles the remaining summation. This algebraic argument completes the proof in the immediate feedback setting and provides a tight upper bound on the term in question \citep{Carpentier2020}. However, the elliptical potential lemma requires that the learner updates the design matrix at the end of every round with the most recent action.\\

This is not the case for the summation in \eqref{equation: terms to bound}, as the feedback associated with the most recent action is not necessarily observable immediately and is, therefore, not used to increment the observed design matrix. Moreover, there will likely be rounds where no feedback arrives at all and rounds where multiple feedbacks return to the learner, meaning that the matrix determinant lemma does not hold; a key argument in the proof. Consequently, we introduce the following technical lemmas that aid in bounding the summation.

\begin{restatable}{lemma}{inverse}\label{lemma: matrix inverse relationship}
Let $\lambda = \alpha\,a(\phi)/\kappa > 0$. Then, $\bar{W}_{t}$ and $\bar{V}_{t}$ are invertible and have inverses that satisfy the following relationship: 
\begin{equation*}
    \bar{W}_{t}^{-1} = \bar{V}_{t}^{-1} + \bar{V}_{t}^{-1} Z_{t}\:\bar{W}_{t}^{-1} = \bar{V}_{t}^{-1} + M_{t}
\end{equation*}
where $M_{t} \coloneqq \bar{V}_{t}^{-1} Z_{t}\:\bar{W}_{t}^{-1}$.
\end{restatable}
\begin{proof}
See Appendix \ref{section: technical lemmas}.
\end{proof}

\begin{restatable}{lemma}{potential}\label{lemma: product elliptical potential}
Let $\{\tau_{t}\}_{t = 1}^{\infty}$ be an arbitrary sequence of non-negative random variables. Then, for $\lambda = \alpha\,a(\phi)/\kappa \geq 1$:
\begin{align*}
    \sum_{t = 1}^{T}\norm{X_{t}}_{M_{t - 1}} &\leq    \sum_{t = 1}^{T}\frac{1 + G_{*} + \tau_{t}}{2}\,\norm{X_{t}}_{\bar{V}_{t - 1}^{-1}}^{2}
\end{align*}
where $G_{*} = \max\{G_{t}: t \leq T\}$.
\end{restatable}
\begin{proof}
See Appendix \ref{section: technical lemmas}.
\end{proof}

Lemma \ref{lemma: matrix inverse relationship} relates the inverse of the observed design matrix to the inverse of the total design matrix and a product of three matrices. This allows us to separate the usual elliptical potential from terms involving the delays by application of the triangle inequality. Then, Lemma \ref{lemma: product elliptical potential} shows that we can relate the remaining summation to a lower-order term. More concretely, 
\begin{align}
    &\sum_{t = 1}^{T}\norm{X_{t}}_{\bar{W}_{t - 1}^{-1}} = \sum_{t = 1}^{T}\norm{X_{t}}_{\bar{V}_{t - 1}^{-1} + M_{t}}\tag{Lemma \ref{lemma: matrix inverse relationship}}\nonumber\\
    &\leq \sum_{t = 1}^{T}(\norm{X_{t}}_{\bar{V}_{t - 1}^{-1}} + \norm{X_{t}}_{M_{t - 1}})\nonumber\tag{Triangle Inequality}\\
    &\leq \sum_{t = 1}^{T} \norm{X_{t}}_{\bar{V}_{t - 1}^{-1}} +  \sum_{t = 1}^{T}\frac{1 + G_{*} + \tau_{t}}{2}\,\norm{X_{t}}_{\bar{V}_{t - 1}^{-1}}^{2} \label{equation: intermediate}
\end{align}
where the final inequality follows from Lemma \ref{lemma: product elliptical potential}. The above reveals that we must bound the number of missing rewards at the end of the $t$-th round and the delay, which we do in the following lemmas. 

\begin{restatable}{lemma}{missing}\label{lemma: number missing bound}
Define $G_{t} = \sum_{s = 1}^{t}\mathds{1}\{s + \tau_{s} > t\}$ and let $\{\tau_{t}\}_{t = 1}^{\infty}$ be a sequence of independent and identically distributed random variables with a finite expectation and define: 
\begin{equation*}
    \psi_{\tau}^{t} \coloneqq \frac{4}{3}\log\left(\frac{3t}{2\delta}\right)+ 2\sqrt{2\mathbb{E}\left[\tau\right]\log\left(\frac{3t}{2\delta}\right)}.
\end{equation*}
Then, $$\mathbb{P}\left(\exists\, t \geq 1: G_{t} \leq   \mathbb{E}[\tau] + \psi_{\tau}^{t}\right) \leq 1 - \delta.$$
\end{restatable}
\begin{proof}
The proof follows the same arguments used in multi-armed bandits \citep{Joulani2013}. However, we include a simple extension to accommodate for continuous delay distributions. See Appendix \ref{section: technical lemmas}.
\end{proof}

\begin{lemma}\label{lemma: delay tail bound}
Let $\{\tau_{t}\}_{t = 1}^{\infty}$ satisfy Assumption \ref{assumption: subexponential delays} and define:
\begin{equation*}
D_{\tau}^{t}  = \min\left\{\sqrt{2 v^2 \log\left(\frac{3t}{2\delta}\right)}, 2b \log\left(\frac{3t}{2\delta}\right)\right\}
\end{equation*}
Then,
$$
\mathbb{P}\left(\exists\, t \geq 1: \tau_{t} \leq \mathbb{E}\left[\tau\right] + D_{\tau}^{t} \right) \leq 1 - \delta
$$
\end{lemma}
\begin{proof}
The above follows from a standard tail bound for subexponential random variables \citep{Wainwright2019} and a union bound.
\end{proof}

Applying Lemmas \ref{lemma: number missing bound} and \ref{lemma: delay tail bound}, combined with the observation that $\psi_{\tau} \coloneqq \psi_{\tau}^{T} \geq \psi_{\tau}^{t}$ and $D_{\tau} \coloneqq D_{\tau}^{T} \geq D_{\tau}^t$ for all $t \leq T$ allows us to bound the delays and the maximum number of missing rewards in Equation \eqref{equation: intermediate} with high probability. By setting $D_{\tau}^{+} = 1 + 2\mathbb{E}[\tau] + D_{\tau} + \psi_{\tau}$, we have that:  
\begin{align*}
    \eqref{equation: intermediate} &\leq \sum_{t = 1}^{T}\norm{X_{t}}_{\bar{V}_{t - 1}^{-1}} + \frac{D_{\tau}^{+}}{2}\sum_{t = 1}^{T}\norm{X_{t}}_{\bar{V}_{t - 1}^{-1}}^{2}\\
    &\leq \sqrt{T\sum_{t = 1}^{T} \norm{X_{t}}_{\bar{V}_{t - 1}^{-1}}^{2}} + \frac{D_{\tau}^{+}}{2}\sum_{t = 1}^{T}\norm{X_{t}}_{\bar{V}_{t - 1}^{-1}}^{2}
\end{align*}
with probability $1 - 2\delta$, where the final inequality follows from an application of Cauchy-Schwarz. Now, the total design matrix is incremented by the most recent action at the end of every round. Therefore, we can apply the elliptical potential lemma, which bounds the remaining summation terms as follows: 
\begin{equation*}
    \sum_{t = 1}^{T}\norm{X_{t}}_{\bar{V}_{t - 1}^{-1}}^{2} \leq 2d\log\left(\frac{d\lambda + T}{d\lambda}\right) = 2dL
\end{equation*}
where $L \coloneqq \log((d\lambda + T)/d\lambda)$. For completeness, we provide a statement and proof of this well-known result in Appendix \ref{section: standard results}. Therefore,
\begin{align*}
    &\sum_{t = 1}^{T}\norm{X_{t}}_{\bar{W}_{t - 1}^{-1}} \leq \sqrt{2dTL} +  dL D_{\tau}^{+}.
\end{align*}
From Equation \eqref{equation: terms to bound}, it is clear that all that remains is to upper bound the width of the confidence set at the end of the final round. Recall $\bar{V}_{t} \succeq \bar{W}_{t}$, because the observed design matrix is a partial sum of positive semi-definite matrices that make up the total design matrix. Therefore, 
\begin{align*}
    \sqrt{\beta_{T}} &\leq \sqrt{\lambda}m_{1} + \frac{R}{\kappa}\sqrt{2\log\left(\frac{\abs{\bar{V}_{T}}^{1/2}}{\lambda^{d/2}}\right) + 2\log\left(\frac{1}{\delta}\right)}\\
    &\leq \sqrt{\lambda} m_{1} + \frac{R}{\kappa}\sqrt{2dL + 2\log\left(\frac{1}{\delta}\right)}
\end{align*}
where the inequality follows from Lemma \ref{lemma: trace-determinant} of Appendix \ref{section: standard results}. Bringing everything together, 
\begin{align*}
    \hat{R}_{T} &\leq 2L_{\mu}\sqrt{\beta_{T}}\sum_{t = 1}^{T}\norm{X_{t}}_{\bar{W}_{t - 1}^{-1}}\\
    &\leq 2L_{\mu}\sqrt{\beta_{T}}\left(\sqrt{2dTL} +  dLD_{\tau}^{+}\right)
\end{align*}
Substituting $D_{\tau}^{+} = 1 + 2\mathbb{E}[\tau] + D_{\tau} + \psi_{\tau}$ and our upper bound on $\sqrt{\beta_{T}}$ into the above, and omitting poly-logarithmic factors gives: 
$$
\widetilde{\mathcal{O}}\left(\frac{d R L_{\mu}}{\kappa}\sqrt{T} +  \frac{d^{3/2} RL_{\mu} \left(\mathbb{E}\left[\tau\right] + \min\left\{v, b\right\}\right)}{\kappa} \right)
$$
completing the proof.
\end{proof}


\begin{remark}\label{remark: expected regret}
Under Assumptions \ref{assumption: subgaussian reward} and \ref{assumption: link function}, one can relax the assumption on the delays from subexponential to only requiring a finite expected value if we only consider a weaker notion of regret, namely the expected regret. Formally, for a fixed $\theta^{*}$ and any delay distribution with a finite expected value:
$$
\mathbb{E}\left[\hat{R}_{T}\right] \leq \widetilde{\mathcal{O}}\left(\frac{d R L_{\mu}}{\kappa}\sqrt{T} + \frac{d^{3/2} RL_{\mu}\mathbb{E}[\tau]}{\kappa}\right)
$$
where we take the expectation over the randomness of the rewards and the delays. This result follows from standard arguments; e.g. by setting $\delta = 1/T$ and using the definition of the confidence sets. Eventually, we end up taking the expectation of Equation \eqref{equation: intermediate} with respect to the rewards and delays.
\end{remark}

\begin{remark}
    In the proof, we focused on the confidence sets given in Lemma \ref{lemma: ofu-glm confidence set}. At the heart of this confidence set is a high probability bound on:
    $$
    \norm*{\sum_{s = 1}^{t}\mathds{1}\{s + \tau_s \leq t\} X_s \eta_s}_{\bar{W}_t^{-1}}
    $$
    which we prove is a non-negative supermartingale under the information structure imposed on the learner by the delays. Many other algorithms utilise slightly different techniques to bound an identical term \citep{Filippi2010, Li2017} or one that differs by the choice of weight in the norm \citep{Faury20a} to define confidence sets. By Lemma \ref{lemma: ofu-glm supermartingale}, Algorithm \ref{algorithm: delayed ofu-glm} ensures that these confidence sets are valid in the delayed feedback setting too. Thus, combining our theoretical results within their analyses will yield a similar additive delay-dependent quantity in the regret bounds under delayed feedback.
\end{remark}

\section{Experimental Results}
We conduct simulated experiments to empirically investigate the impact of delayed feedback in  Linear and Logistic bandits. We compare our algorithmic ideas to other approaches for the setting of delayed feedback in generalised linear bandits, which inflate the exploration bonus by the number of missing rewards \citep{Blanchet2020}.\\

In our experiments, we consider $d \in \{5, 10, 20\}$ and fix $T = 100,000$. At the start of the simulations, we randomly sample $\theta^{*}$ from the unit ball for the Linear and Logistic bandit environments so that it remains fixed across each independent run of our experiments. The decision set in each round is a random sample of $K = 100$ actions from the unit ball. We choose the confidence parameter for each algorithm so that the theoretical guarantees hold with probability $0.95$ by setting $\delta = 0.05/3$. All results are averaged over $30$ independent runs and the shaded region in all figures represent the standard errors of the estimates.\\

We consider several delay distributions to investigate the impact of the delays on the performance of each algorithm, namely: 
\begin{itemize}
    \item Exponential$(\lambda)$ with $\lambda = 1/\mathbb{E}[\tau]$,
    \item Uniform$(a, b)$ with $a = 0$ and $b = 2\mathbb{E}[\tau]$,
    \item Pareto$(a, x_m = 1)$ with $a = (1 + \mathbb{E}[\tau])/\mathbb{E}[\tau]$.
\end{itemize}

For each delay distribution, we consider expected values of  $\mathbb{E}[\tau] = \{100, 250, 500, 1000\}$. Notably, Assumption \ref{assumption: subexponential delays} holds for the uniform and exponential distributions. However, it does not hold for the Pareto distribution. \citet{Blanchet2020} make a similar subexponential assumption on the delays, meaning that their theoretical guarantees do not hold for Pareto delays either.\\ 

\begin{figure*}[h!]
    \centering
    \includegraphics[width = \textwidth]{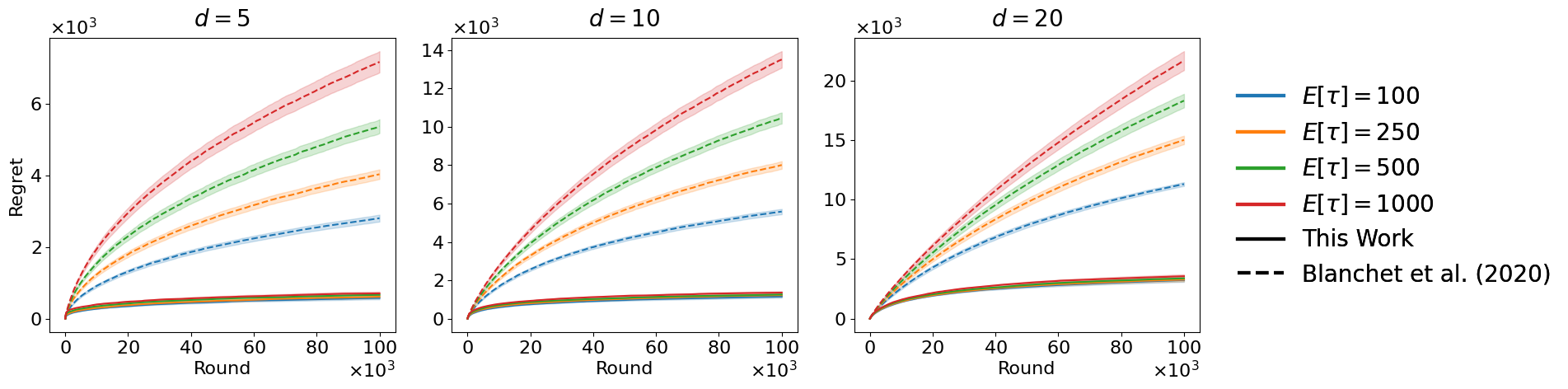}
    \caption{Linear Bandit \& Exponentially Distributed Delays.}
    \label{figure: linear bandit}
\end{figure*}

\begin{figure*}[h!]
    \centering
    \includegraphics[width = \textwidth]{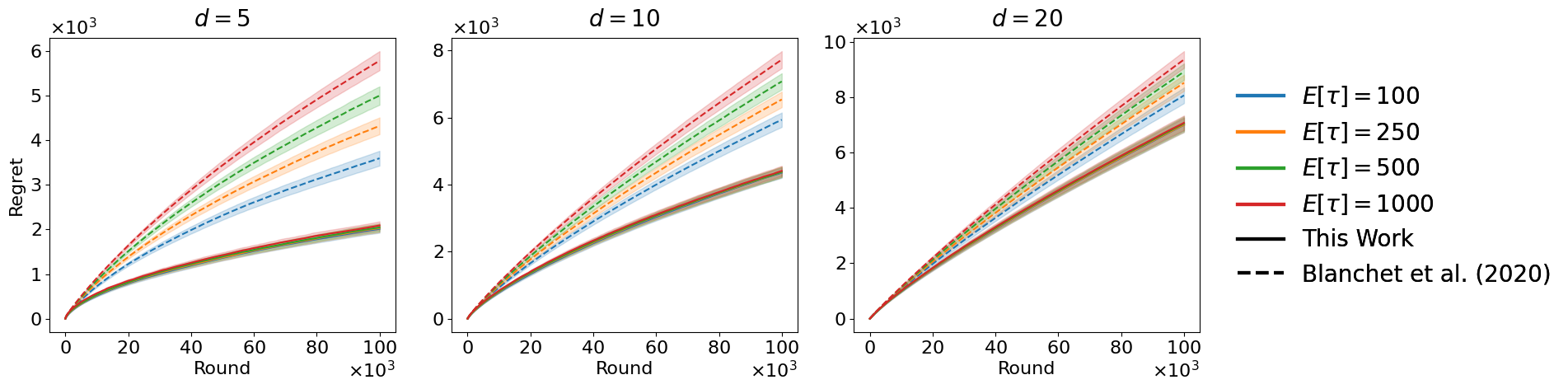}
    \caption{Logistic Bandit \& Exponentially Distributed Delays.}
    \label{figure: logistic bandit}
\end{figure*}

Figures \ref{figure: linear bandit} and \ref{figure: logistic bandit} illustrate the results of our experiments for exponentially distributed delays. Appendix \ref{sec: additional experiments} shows similar results for the other delay distributions and expected delays considered. The empirical results show that our approach out-performs existing algorithms designed for the same problem setting. These results are consistent with the theoretical guarantees, where the delayed feedback causes an additive penalty for our algorithm and a larger multiplicative penalty for the approach of \citet{Blanchet2020}.\\

\begin{figure*}[h!]
    \centering
    \includegraphics[width = \textwidth]{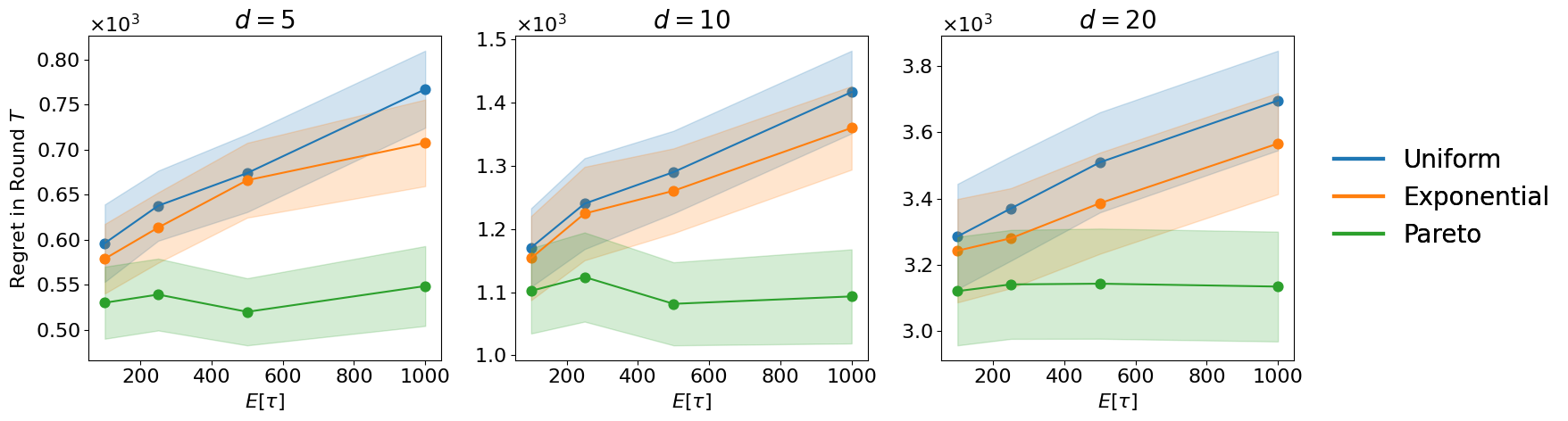}
    \caption{Final Round Regret vs. Expected Delay in Linear Bandits.}
    \label{figure: linear bandit delay}
\end{figure*}

\begin{figure*}[h!]
    \centering
    \includegraphics[width = \textwidth]{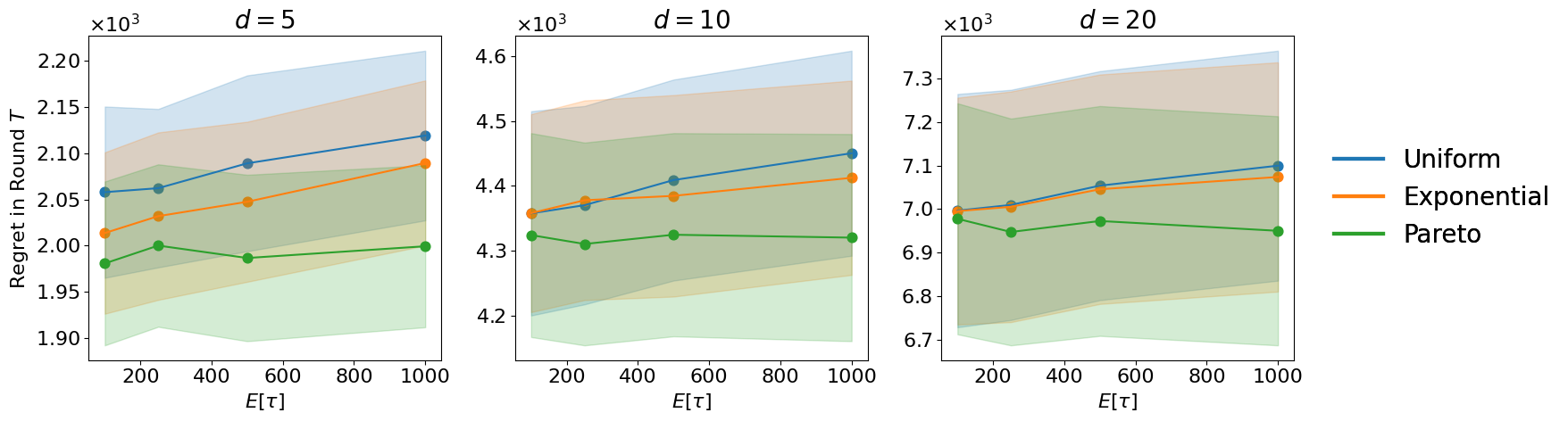}
    \caption{Final Round Regret vs. Expected Delay in Logistic Bandits.}
    \label{figure: logistic bandit delay}
\end{figure*}

Figures \ref{figure: linear bandit delay} and \ref{figure: logistic bandit delay} show the regret at the end of the final round as a function of the expected delay for our algorithm. Although Assumption \ref{assumption: subexponential delays} does not hold for delays drawn from the Pareto distribution, our algorithm still provides good performance for the various values of $\mathbb{E}[\tau]$ considered by our experiments. Notably, these empirical results are consistent with the expected regret guarantee stated in Remark \ref{remark: expected regret}, which only requires that the delays have a finite expected value. Our experiments also suggest that the penalty for Pareto delays is lesser than the other distributions under investigation. This observation may be due to our particular parameterisation of the Pareto distribution producing many small delays; indeed, $\mathbb{P}(\tau_t \leq 20) \geq 0.95$ for all the expected delays considered by our experiments. The same is not true for the other distributions.

\section{Conclusion}

In this work, we studied the impact of delayed feedback on algorithms for generalised linear bandits. Under Assumption \ref{assumption: subexponential delays}, that the delays are subexponential random variables, we designed an optimistic algorithm whose worst-case regret bound increases by an additive term involving the expected delay. We obtain a similar result for the expected regret, which only requires that the delays have a finite expected value.\\

These theoretical results significantly improve on prior work, where existing algorithms suffer a multiplicative penalty and require a-priori knowledge of the delay distribution as input. Reducing the delay dependence from multiplicative to additive was possible by introducing a novel technique to carefully separate the delays from the difficulty of the learning problem. Doing so allowed us to define tighter confidence sets than existing algorithms, leading to better theoretical guarantees and superior empirical performance. Indeed, the theoretical techniques introduced in this paper might be useful in other bandit problems with complex feedback structures.\\

Our result nearly recovers the additive delay penalty from multi-armed bandits, despite the additional difficulties of our setting. Whether or not it is possible to remove the $d$-dependence entirely remains an interesting open question. Another open question relates to relaxing our assumptions on the delays. Namely, can we get high probability bounds that only require that the delays have a finite expected value? We anticipate that addressing these open questions may require adjustments to our theoretical techniques or different algorithmic approaches.\\

Finally, we expect that similar results hold for a Thompson Sampling version of our algorithm. Combining techniques found in \cite{Russo2014} with those in this paper will likely give similar guarantees for the Bayesian regret.

\bibliographystyle{plainnat}
\bibliography{references} 

\newpage
\begin{appendix}
\section{Confidence Sets}\label{section: confidence sets}
Here, we show that the confidence sets are valid under delayed feedback. To that end, we define the following $\sigma$-algebra:
\begin{equation}\label{equation: sigma-algebra}
    \mathcal{F}_{t - 1} = \sigma \left(\left\{\left(X_s, C_s^{t - 1}, Y_s C_s^{t - 1}\,\right): s \leq t - 1\right\}\,\cup \mathcal{A}_{t}\right)
\end{equation}
Consequently, $Y_{t}$ is $\mathcal{F}_{t}$-measurable. Further, $X_{t}$ is $\mathcal{F}_{t - 1}$-measurable. For notational purposes, we find it useful to define:
\begin{equation*}
    g_{t}\left(\theta\right) = \alpha \,a(\phi)\, \theta +  \sum_{s = 1}^{t} \mathds{1}\left\{s + \tau_{s} \leq t\right\} \mu\left(\langle X_{t}, \theta\,\rangle\right)X_{s}
\end{equation*}
as well as the second derivative of the negative log likelihood: 
\begin{equation*}
    H_{t}\left(\theta\right) = \alpha\,a(\phi) + \sum_{s = 1}^{t} \mathds{1}\left\{s + \tau_{s} \leq t\right\} \dot{\mu}\left(\langle X_{s}, \theta\,\rangle\right)X_{s}X_{s}^{T}
\end{equation*}

Now, $\hat{\theta}_{t}$ is the vector satisfying the following equality:
\begin{align}
    \frac{\partial\mathcal{L}_{t}\left(\theta, \alpha\right)}{\partial \theta} &= \left[\sum_{s = 1}^{t} \mathds{1}\left\{s + \tau_{s} \leq t\right\} \left(Y_{s} - \mu\left(X_{t}^{T} \theta\right)\right) X_{s}\right] - \alpha \,a(\phi)\, \theta\nonumber\\
    &= \left[\sum_{s = 1}^{t}\mathds{1}\left\{s + \tau_{s} \leq t\right\} Y_{s}X_{s}\right] - g_{t}\left(\theta\right)\nonumber\\
    &= 0\label{equation: estimator}
\end{align}

\confset*
\begin{proof}
By Lemmas \ref{lemma: mean value theorem} and  \ref{lemma: separate bias contribution} of Appendix \ref{section: confidence set supporting lemmas}, we have that: $$
\norm{\hat{\theta}_{t} - \theta^{*}}_{\bar{W}_{t}} \leq \sqrt{\lambda}\norm{\theta^{*}}_{2} + \frac{1}{\kappa}\norm{S_{t}}_{\bar{W}_{t}^{-1}}
$$
where 
$$
S_{t} = \sum_{s = 1}^{t}\mathds{1}\left\{s + \tau_{s} \leq t\right\} X_{s}\eta_{s}
$$

Further Lemma \ref{lemma: ofu-glm supermartingale} of \ref{section: confidence set supporting lemmas} reveals that the last term in the above is a non-negative supermartingale under the delayed feedback information structure. Therefore, we are able to use known methods for bounding self-normalised vector-valued martingales \citep{Pena2004, Abbasi-Yadkori2011}. Let $\omega$ be a stopping time with respect to the filtration. Applying Lemma 9 of \cite{Abbasi-Yadkori2011} to the stopped martingale gives:
\begin{align*}
    \mathbb{P}\left( \norm{S_{\omega}}_{\bar{W}_{\omega}^{-1}} \geq R\sqrt{\log\left(\frac{\det\left(\bar{W}_{\omega}\right)}{\lambda^{d}}\right) + 2\log\left(\frac{1}{\delta}\right)}\,\right)
    \leq \delta
\end{align*}

Since Lemma \ref{lemma: ofu-glm supermartingale} guarantees that the stopped supermartingale is well-defined, regardless of whether the stopping time is finite, the above inequality holds across all rounds without the need for a union bound. That is: 
\begin{equation}\label{equation: glm-ucb martingale bound}
    \mathbb{P}\left(\exists\, t \geq 0: \norm{S_{t}}_{\bar{W}_{t}^{-1}} \geq R\sqrt{\log\left(\frac{\det\left(\bar{W}_{t}\right)}{\lambda^{d}}\right) + 2\log\left(\frac{1}{\delta}\right)}\,\right)
    \leq \delta
\end{equation}
Therefore, with probability at least $1 - \delta$:
\begin{align*}
    \norm{\hat{\theta}_{t} - \theta^{*}}_{\bar{W}_{t}} &\leq \sqrt{\lambda}\norm{\theta^{*}}_{2} + \frac{1}{\kappa}\norm{ S_{t}}_{\bar{W}_{t}^{-1}}\tag{Lemmas \ref{lemma: mean value theorem} \& \ref{lemma: separate bias contribution}}\\
    &\leq \sqrt{\lambda}\norm{\theta^{*}}_{2} + \frac{R}{\kappa}\sqrt{\log\left(\frac{\det\left(\bar{W}_{t}\right)}{\lambda^{d}}\right) + 2\log\left(\frac{1}{\delta}\right)}
\end{align*}
as required. 
\end{proof}

\subsection{Supporting Lemmas}\label{section: confidence set supporting lemmas}
Proving Lemma \ref{lemma: ofu-glm confidence set} requires several supporting lemmas. Firstly, the confidence sets are in terms of $\hat{\theta}_{t}$ and $\theta^{*}$. Conversely, Equation \eqref{equation: estimator} reveals that our estimation procedure involves $g_{t}(\hat{\theta})$ and $g_{t}(\theta^{*})$. The following lemma allowed us to relate these two quantities to one another. 

\begin{lemma}\label{lemma: mean value theorem}
Let $\theta_{1}\in\mathbb{R}^{d}$ and $\theta_{2}\in\mathbb{R}^{d}$ be arbitrary vectors, and $\lambda = \alpha a(\phi)/\kappa$. Then, the following inequality holds:
\begin{equation*}
    \kappa \,\norm{\theta_{1} - \theta_{2}}_{\bar{W}_{t}} \leq \norm{g\left(\theta_{1}\right) - g\left(\theta_{2}\right)}_{\bar{W}_{t}^{-1}}
\end{equation*}
\end{lemma}
\begin{proof}
Similarly to \citet{Filippi2010}, we apply the mean value theorem to the terms inside the norm on the right-hand side of the above, which allows us to related them to the original vectors. Expanding $g\left(\theta_{1}\right)$ and $g\left(\theta_{2}\right)$ reveals that: 
\begin{align}
    g\left(\theta_{1}\right) - g\left(\theta_{2}\right) &= \alpha \,a(\phi)\,\theta_{1} - \alpha \,a(\phi)\, \theta_{2} +  \sum_{s = 1}^{t} \mathds{1}\left\{s + \tau_{s} \leq t\right\}\left[ \,\mu\left(\langle X_{t}, \theta_{1}\,\rangle\right) - \mu\left(\langle X_{t}, \theta_{2}\,\rangle\right) \right]X_{s}\nonumber\\
    &= \alpha \,a(\phi)\,\theta_{1} - \alpha \,a(\phi)\, \theta_{2} +  \sum_{s = 1}^{t} \mathds{1}\left\{s + \tau_{s} \leq t\right\} \,\dot{\mu}\left(\langle X_{t}, \bar{\theta}\,\rangle\right) X_{s}X_{s}^{T}\left(\theta_{1} - \theta_{2}\right)\nonumber\\
    &= \left[\alpha\,a(\phi) + \sum_{s = 1}^{t} \mathds{1}\left\{s + \tau_{s} \leq t\right\} \,\dot{\mu}\left(\langle X_{t}, \bar{\theta}\,\rangle\right) X_{s}X_{s}^{T}\right]\left(\theta_{1} - \theta_{2}\right)\nonumber\\
    &= H_{t}\left(\bar{\theta}\,\right)\left(\theta_{1} - \theta_{2}\right)\label{equation: mean value theorem}
\end{align} 
where the second equality follows from the mean value theorem for some $\bar{\theta}\in(\theta_{2}, \theta_{1})$. Rewriting the Hessian for some $\theta\in \mathbb{R}^{d}$ and recalling that $\kappa \leq \dot{\mu}(z)$ reveals that: 
\begin{equation}
    H_{t}\left(\theta\right) = \alpha\, a(\phi) + \sum_{s = 1}^{t} \mathds{1}\left\{s + \tau_{s} \leq t\right\} \dot{\mu}\left(\langle X_{s}, \theta\,\rangle\right)X_{s}X_{s}^{T}
    \succeq  
    \kappa \left[\frac{\alpha \,a(\phi)}{\kappa} + \sum_{s = 1}^{t} \mathds{1}\left\{s + \tau_{s} \leq t\right\}X_{s}X_{s}^{T}\right] = \kappa \bar{W}_{t} \label{equation: hessian design relationship}
\end{equation}
From Equation \eqref{equation: hessian design relationship}, we immediately have that $H_{t}^{-1}(\theta) \preceq \bar{W}_{t}^{-1}/\kappa$. Combining Equation \eqref{equation: mean value theorem} with the partial ordering of Equation \eqref{equation: hessian design relationship} gives: 
\begin{align*}
    \norm*{\theta_{1} - \theta_{2}}_{\kappa \bar{W}_{t}} &\leq \norm*{\theta_{1} - \theta_{2}}_{H_{t}\left(\bar{\theta}\right)} \tag{$\kappa \bar{W}_{t} \preceq H_{t}(\theta)$}\\
    &= \norm*{H_{t}^{1/2}\left(\bar{\theta}\right)\left(\theta_{1} - \theta_{2}\right)}_{2} \tag{$\norm{x}_{A} = \norm{A^{1/2} x}_{2}$}\\
    &= \norm*{H_{t}^{-1/2}\left(\bar{\theta}\right)\left(g_{t}\left(\theta_{1}\right) - g_{t}\left(\theta_{2}\right)\right)}_{2} \tag{Equation \eqref{equation: mean value theorem}}\\
    &= \norm*{g_{t}\left(\theta_{1}\right) - g_{t}\left(\theta_{2}\right)}_{H_{t}^{-1}\left(\bar{\theta}\right)}\tag{$\norm{A^{1/2} x}_{2} = \norm{x}_{A}$}\\
    &\leq \norm*{g_{t}\left(\theta_{1}\right) - g_{t}\left(\theta_{2}\right)}_{\frac{1}{\kappa}\bar{W}_{t}^{-1}} \tag{$H_{t}^{-1}(\theta)\preceq \bar{W}_{t}^{-1}/\kappa $}
\end{align*}
Therefore, using homogeneity property of norms on the first and last terms of the above reveals that: 
\begin{equation*}
    \sqrt{\kappa}\norm*{\theta_{1} - \theta_{2}}_{\bar{W}_{t}} = \norm*{\theta_{1} - \theta_{2}}_{\kappa \bar{W}_{t}} \leq \norm*{g_{t}\left(\theta_{1}\right) - g_{t}\left(\theta_{2}\right)}_{\frac{1}{\kappa}\bar{W}_{t}^{-1}} = \frac{1}{\sqrt{\kappa}} \norm*{g_{t}\left(\theta_{1}\right) - g_{t}\left(\theta_{2}\right)}_{\bar{W}_{t}^{-1}}
\end{equation*}
Bringing all $\kappa$'s to the left hand side side gives the stated result.
\end{proof}

\begin{lemma}\label{lemma: separate bias contribution}
Let $\theta^{*}$ be the unknown parameter of the environment and $\hat{\theta}_{t}$ be the solution to \eqref{equation: estimator}. Further, define $\lambda = \alpha\,a(\phi)/\kappa$ and 
$$
S_{t} = \sum_{s = 1}^{t}\mathds{1}\left\{s + \tau_{s} \leq t\right\} X_{s}\eta_{s}
$$
Then, 
\begin{equation*}
    \norm{\hat{\theta}_{t} - \theta^{*}}_{\bar{W}_{t}} \leq \sqrt{\lambda}\norm{\theta^{*}}_{2} + \frac{1}{\kappa}\norm{S_{t}}_{\bar{W}_{t}^{-1}}
\end{equation*}
\end{lemma}
\begin{proof}
By Lemma \ref{lemma: mean value theorem}, we have that: 
\begin{equation}
    \norm{\hat{\theta}_{t} - \theta^{*}}_{\bar{W}_{t}} \leq \frac{1}{\kappa} \,\norm*{\,g\left(\hat{\theta}_{t}\right) - g\left(\theta^{*}\right)}_{\bar{W}_{t}^{-1}}\label{equation: mean value theorem result}
\end{equation}
Since $\hat{\theta}_{t}$ is the solution to \eqref{equation: estimator}, it follows that: 
\begin{align*}
\left[\sum_{s = 1}^{t}\mathds{1}\left\{s + \tau_{s} \leq t\right\} Y_{s}X_{s}\right] - g_{t}\left(\hat{\theta}_{t}\right) = 0 \iff  g_{t}\left(\hat{\theta}_{t}\right) = \left[\sum_{s = 1}^{t}\mathds{1}\left\{s + \tau_{s} \leq t\right\} Y_{s}X_{s}\right] 
\end{align*}

Substituting the above into \eqref{equation: mean value theorem result} gives: 
\begin{align*}
    \norm*{\hat{\theta}_{t} - \theta^{*}}_{\bar{W}_{t}} &\leq \frac{1}{\kappa} \,\norm*{\,g\left(\hat{\theta}_{t}\right) - g\left(\theta^{*}\right)}_{\bar{W}_{t}^{-1}}\\
    &= \frac{1}{\kappa} \,\norm*{\sum_{s = 1}^{t} \mathds{1}\left\{s + \tau_{s} \leq t\right\} Y_{s}X_{s} - g\left(\theta^{*}\right)}_{\bar{W}_{t}^{-1}}\\
    &= \frac{1}{\kappa} \,\norm*{\sum_{s = 1}^{t} \mathds{1}\left\{s + \tau_{s} \leq t\right\} Y_{s}X_{s} - \left[\alpha\,a(\phi)\,\theta^{*} +  \sum_{s = 1}^{t} \mathds{1}\left\{s + \tau_{s} \leq t\right\} \mu\left(\langle X_{t}, \theta^{*}\,\rangle\right)X_{s}\right]}_{\bar{W}_{t}^{-1}}\\
    &= \frac{1}{\kappa} \,\norm*{-\alpha\,a(\phi)\,\theta^{*} +  \sum_{s = 1}^{t} \mathds{1}\left\{s + \tau_{s} \leq t\right\}\left[ \,Y_{s} - \mu\left(\langle X_{s}, \theta^{*}\,\rangle\right)\right]X_{s}}_{\bar{W}_{t}^{-1}}\\
    &= \frac{1}{\kappa} \,\norm*{-\alpha\,a(\phi)\,\theta^{*} +  \sum_{s = 1}^{t} \mathds{1}\left\{s + \tau_{s} \leq t\right\}\left[ \,\mu\left(\langle X_{s}, \theta^{*}\,\rangle\right) + \eta_{s} -  \mu\left(\langle X_{s}, \theta^{*}\,\rangle\right)  \right]X_{s}}_{\bar{W}_{t}^{-1}}\\
    &= \frac{1}{\kappa} \,\norm*{-\alpha\,a(\phi)\,\theta^{*} + \sum_{s = 1}^{t} \mathds{1}\left\{s + \tau_{s} \leq t\right\} X_{s}\eta_{s} }_{\bar{W}_{t}^{-1}} =  \frac{1}{\kappa} \,\norm*{S_{t} - \alpha\phi\theta^{*}}_{\bar{W}_{t}^{-1}}\\
    &\leq \frac{1}{\kappa} \,\norm*{\alpha\,a(\phi)\,\theta^{*}}_{\bar{W}_{t}^{-1}} + \frac{1}{\kappa}\norm{S_{t}}_{\bar{W}_{t}^{-1}}\\
    &\leq \frac{1}{\kappa}\sqrt{\frac{\alpha^2 \,a(\phi)^2}{\lambda}}\norm{\theta^{*}}_{2} + \frac{1}{\kappa}\norm{S_{t}}_{\bar{W}_{t}^{-1}}\\
    &= \sqrt{\lambda}\norm{\theta^{*}}_{2} + \frac{1}{\kappa}\norm{S_{t}}_{\bar{W}_{t}^{-1}}
\end{align*}
where the final inequality follows from the fact that $\bar{W}_{t}^{-1} \preceq\lambda^{-1}I$, and the final equality follows from the fact that $\lambda = \alpha\,a(\phi)/\kappa$.
\end{proof}

All that remains is bounding the norm involving the noise terms with high probability, which is the second term in the result stated in Lemma \ref{lemma: separate bias contribution}. By Fenchel Duality, we have that \citep{Abbasi-Yadkori2011}: 
\begin{align}
    \frac{1}{2}\norm{S_{t}}_{\bar{W}_{t}^{-1}}^{2} &= \max_{x\in \mathbb{R}^{d}}\left\{\langle x, S_{t}\rangle - \frac{1}{2}\norm{x}_{\bar{W}_{t}}^{2}\right\}\nonumber\\
    &= \max_{x\in \mathbb{R}^{d}}\left\{\log\left(\exp\left(\langle x, S_{t}\rangle - \frac{1}{2}\norm{x}_{\bar{W}_{t}}^{2}\right)\right)\right\}\nonumber\\
    &= \log\left(\max_{x\in \mathbb{R}^{d}}\left\{\exp\left(\langle x, S_{t}\rangle - \frac{1}{2}\norm{x}_{\bar{W}_{t}}^{2}\right)\right\}\right)\label{equation: confidence set validity 2}
\end{align}

Equation \eqref{equation: confidence set validity 2} suggests that it would be useful to obtain a high probability bound on the following random variable: 
\begin{equation*}
    M_{t}\left(x\right) = \exp\left(\frac{1}{R}\langle x, S_{t}\rangle - \frac{1}{2}\norm{x}_{W_{t}}^{2}\right)
\end{equation*}
where 
$$
W_{t} = \bar{W}_{t} - \lambda I = \sum_{s = 1}^{t}\mathds{1}\left\{s + \tau_{s} \leq t\right\}X_{s}X_{s}^{T}
$$
for an arbitrary vector $x\in\mathbb{R}^{d}$. To do so, we first establish the following supermartingale argument, which is essential in showing the validity of the confidence sets. Due to the delayed feedback, we cannot directly use results from the immediate feedback setting. Therefore, we make the necessary adjustments to account for the delays.


\begin{lemma}\label{lemma: ofu-glm supermartingale}
Let $x \in \mathbb{R}^{d}$ be an arbitrary vector and define: 
\begin{equation*}
    M_{t}\left(x\right) = \exp\left(\frac{1}{R}\langle x, S_{t}\rangle - \frac{1}{2}\norm{x}_{W_{t}}^{2}\right) = \exp\left(\sum_{s = 1}^{t}\mathds{1}\left\{s + \tau_{s} \leq t\right\}\left(\frac{\langle x, X_{s}\rangle\,\eta_{s}}{R} - \frac{1}{2}\langle x, X_{s}\rangle^{2}\right)\right)
\end{equation*}
Let $\omega$ be a stopping time with respect to the filtration $\{\mathcal{F}_{t}\}_{t = 0}^{\infty}$. Then, $M_{\omega}(x)$ is almost surely well-defined and $\mathbb{E}\left[M_{\omega}(x)\right] \leq 1$.
\end{lemma}
\begin{proof}
Recall that $C_s^{t} = \mathds{1}\left\{s + \tau_s \leq t\right\}$. We start by re-writing $M_{t}(x)$ in terms of $M_{t - 1}(x)$:
\begin{align*}
    &M_{t}\left(x\right) = \exp\left(\,\sum_{s = 1}^{t}\mathds{1}\left\{s + \tau_{s} \leq t\right\}\left(\frac{\langle x, X_{s}\rangle\,\eta_{s}}{R} - \frac{1}{2}\langle x, X_{s}\rangle^{2}\right)\right)\\
    &= \exp\left(\,\sum_{s = 1}^{t}C_{s}^{t - 1}\left(\frac{\langle x, X_{s}\rangle\,\eta_{s}}{R} - \frac{1}{2}\langle x, X_{s}\rangle^{2}\right) + \sum_{s = 1}^{t}\mathds{1}\left\{s + \tau_{s} > t - 1\right\}\left(\frac{\langle x, X_{s}\rangle\,\eta_{s}}{R} - \frac{1}{2}\langle x, X_{s}\rangle^{2}\right)\right)\\
    &= \exp\left(\,\sum_{s = 1}^{t}C_s^{t - 1}\left(\frac{\langle x, X_{s}\rangle\,\eta_{s}}{R} - \frac{1}{2}\langle x, X_{s}\rangle^{2}\right)\right) \exp\left(\,\sum_{s = 1}^{t}\mathds{1}\left\{s + \tau_{s} > t - 1\right\}\left(\frac{\langle x, X_{s}\rangle\,\eta_{s}}{R} - \frac{1}{2}\langle x, X_{s}\rangle^{2}\right)\right)\\
    &= \exp\left(\,\sum_{s = 1}^{t - 1}C_s^{t - 1}\left(\frac{\langle x, X_{s}\rangle\,\eta_{s}}{R} - \frac{1}{2}\langle x, X_{s}\rangle^{2}\right)\right) \exp\left(\,\sum_{s = 1}^{t}\mathds{1}\left\{s + \tau_{s} > t - 1\right\}\left(\frac{\langle x, X_{s}\rangle\,\eta_{s}}{R} - \frac{1}{2}\langle x, X_{s}\rangle^{2}\right)\right)\\
    &= M_{t - 1}\left(x\right) \exp\left(\,\sum_{s = 1}^{t}\mathds{1}\left\{s + \tau_{s} > t - 1\right\}\left(\frac{\langle x, X_{s}\rangle\,\eta_{s}}{R} - \frac{1}{2}\langle x, X_{s}\rangle^{2}\right)\right)
\end{align*}
where the penultimate equality follows from the fact that $\mathds{1}\{t + \tau_t \leq t - 1\} = 0$ as the delays are non-negative random variables, allowing us to stop the first summation at round $t - 1$ by pulling the corresponding $\exp(0)$ out of the the summation and utilising that $\exp(a + b) = \exp(a)\exp(b)$.\\

Recall that $C_s^{t - 1} = \mathds{1}\left\{s + \tau_s \leq t - 1\right\}$ is $\mathcal{F}_{t - 1}$-measurable. Since $\mathcal{F}_{t - 1}$ is a $\sigma$-algebra, $\mathds{1}\left\{s + \tau_s > t - 1\right\}$ must also be measurable. Utilising this fact, it is clear that everything except the noise terms are $\mathcal{F}_{t - 1}$-measurable. Assumption \ref{assumption: subgaussian reward} guarantees the noise is subgaussian, therefore: 
\begin{align*}
&\mathbb{E}\left[M_{t}\left(x\right)\,\vert\,\mathcal{F}_{t - 1}\right] =  \mathbb{E}\left[ M_{t - 1}\left(x\right) \,\exp\left(\,\sum_{s = 1}^{t}\left(1 - C_{s}^{t - 1}\right)\left(\frac{\langle x, X_{s}\rangle\,\eta_{s}}{R} - \frac{1}{2}\langle x, X_{s}\rangle^{2}\right)\right) \,\Big\vert\,\mathcal{F}_{t - 1}\right]\\
&= M_{t - 1}\left(x\right)\,\mathbb{E}\left[\exp\left(\,\sum_{s = 1}^{t}\left(1 - C_{s}^{t - 1}\right)\left(\frac{\langle x, X_{s}\rangle\,\eta_{s}}{R} - \frac{1}{2}\langle x, X_{s}\rangle^{2}\right)\right)\,\Big\vert\,\mathcal{F}_{t - 1}\right]\\
&=M_{t - 1}\left(x\right)\,\exp\left(\,-\sum_{s = 1}^{t}\frac{\left(1 - C_{s}^{t - 1}\right)\langle x, X_{s}\rangle^{2}}{2}\right)\,\mathbb{E}\left[\exp\left(\,\sum_{s = 1}^{t}\frac{\left(1 - C_{s}^{t - 1}\right) \langle x, X_{s}\rangle\,\eta_{s}}{R}\right) \,\Big\vert\,\mathcal{F}_{t - 1}\right]\\
&\leq M_{t - 1}\left(x\right)\,\exp\left(\,-\frac{1}{2}\sum_{s = 1}^{t}\left(1 - C_{s}^{t - 1}\right)\langle x, X_{s}\rangle^{2}\right)\, \exp\left(\frac{1}{2}\sum_{s = 1}^{t}\left(1 - C_{s}^{t - 1}\right)\langle x, X_{s} \rangle^{2}\right)\\
    &= M_{t - 1}\left(x\right)
\end{align*}
showing that $\{M_{t}(x)\}_{t = 0}^{\infty}$ is indeed a non-negative supermartingale.\footnote{By definition, the exponential function is always positive. Hence, $M_{t}(x)$ is always non-negative.} For conciseness, denote: 
$$
P_{t} = \exp\left(\,\sum_{s = 1}^{t}\mathds{1}\left\{s + \tau_{s} > t - 1\right\}\left(\frac{\langle x, X_{s}\rangle\,\eta_{s}}{R} - \frac{1}{2}\langle x, X_{s}\rangle^{2}\right)\right)
$$
Then, by the law of  total expectation: 
\begin{align*}
    \mathbb{E}\left[M_{t}\left(x\right)\right] &= \mathbb{E}\left[M_{t - 1}\left(x\right) P_{t}\right] 
    = \mathbb{E}\left[\mathbb{E}\left[M_{t - 1}\left(x\right) P_{t}\,\vert\, \mathcal{F}_{t - 1}\right]\right] 
    = \mathbb{E}\left[M_{t - 1}\left(x\right)\,\mathbb{E}\left[ P_{t}\,\vert\, \mathcal{F}_{t - 1}\right]\right]\\
    &\leq \mathbb{E}\left[M_{t - 1}\left(x\right)\right]
    = \mathbb{E}\left[M_{t - 2}\left(x\right) P_{t - 1}\right]
    = \mathbb{E}\left[M_{t - 2}\left(x\right)\,\mathbb{E}\left[ P_{t - 1}\,\vert\, \mathcal{F}_{t - 2}\right]\right]\\
    &\:\:\,\vdots \\
    &\leq \mathbb{E}\left[M_{1}\left(x\right)\right] = \mathbb{E}\left[\mathbb{E}\left[M_{1}\left(x\right)\,\vert\,\mathcal{F}_{0}\right]\right]\\
    &\leq 1
\end{align*}
where we define $M_{0}(x) = 1$. By the convergence theorem for non-negative supermartingales: 
$$
M_{\infty}^{x} = \lim_{t \rightarrow \infty }M_{t}^{x}
$$
is almost surely well-defined. Hence, $M_{\omega}^{x}$ is almost surely well-defined, regardless of whether the stopping time is finite or not. Now, Fatou's Lemma tells us that: 
\begin{align*}
    \mathbb{E}\left[M_{\omega}\left(x\right)\right] = \mathbb{E}\left[\liminf_{t \rightarrow \infty} M_{\min\left\{t, \omega\right\}}\left(x\right)\right] \leq \liminf_{t \rightarrow \infty} \mathbb{E}\left[M_{\min\left\{t, \omega\right\}}\left(x\right)\right]
\end{align*}

Combining the right hand side of the above with the law of total expectation reveals that for any $t \geq 0$: 
\begin{align*}
    \mathbb{E}\left[M_{\min\left\{t, \omega\right\}}\left(x\right)\right] &= \mathbb{E}\left[\mathbb{E}\left[M_{\min\left\{t, \omega\right\}}\left(x\right)\,\vert\, \mathcal{F}_{t - 1}\right]\right]\\
    &\leq \mathbb{E}\left[M_{\min\left\{t - 1, \omega\right\}}\left(x\right)\right] = \mathbb{E}\left[\mathbb{E}\left[M_{\min\left\{t - 1, \omega\right\}}\left(x\right)\,\vert\, \mathcal{F}_{t - 1}\right]\right]\\
    &\:\:\,\vdots\\
    &\leq \mathbb{E}\left[M_{\min\left\{0, \omega\right\}}\left(x\right)\right] = \mathbb{E}\left[M_{0}\left(x\right)\right] = 1
\end{align*}
Therefore, $\mathbb{E}[M_{\omega}^{x}] \leq 1$, as required. 
\end{proof}

\section{Missing Proofs}\label{section: technical lemmas}
Proving Theorem \ref{theorem: delayed ofu-glm} required the introduction of four technical lemmas. These lemmas are crucial in showing that our algorithms only suffer from an additive penalty due to the delays. 

\inverse*
\begin{tproof}
From Equations \eqref{equation: total design matrix}, \eqref{equation: observed design matrix} and \eqref{equation: missing design matrix}, and $\lambda I\succ 0$, we have that the total and observed gram matrices are symmetric and positive-definite. They are symmetric because they are the sum of symmetric matrices. And they are positive-definite because they are the sum of a positive-definite matrix and a positive semi-definite matrix. Thus, the first part of the lemma follows from the fact that all symmetric positive-definite matrices are invertible.\\

Next, we move on to the second claim of the lemma. From Equations \eqref{equation: total design matrix}, \eqref{equation: observed design matrix} and \eqref{equation: missing design matrix}, we have that the total, observed, and missing design matrices satisfy the following relationship: 
\begin{equation}\label{equation: matrices relationship}
    \bar{V}_{t} = \lambda I + \sum_{s = 1}^{t} X_{s}X_{s}^{T}\left(\mathds{1}\{s + \tau_{s} \leq t\} + \mathds{1}\{s + \tau_{s} > t\}\right)
    = \bar{W}_{t} + Z_{t}
\end{equation}
We prove the second statement in the lemma as follows:
\begin{align*}
    \bar{W}_{t}^{-1} &= \bar{V}_{t}^{-1} + \bar{W}_{t}^{-1} - \bar{V}_{t}^{-1}\\
    &= \bar{V}_{t}^{-1} + \bar{V}_{t}^{-1} \bar{V}_{t}\bar{W}_{t}^{-1} - \bar{V}_{t}^{-1}\bar{W}_{t}\bar{W}_{t}^{-1} \\
    &= \bar{V}_{t}^{-1} + \bar{V}_{t}^{-1}\left( \bar{V}_{t} - \bar{W}_{t}\right)\bar{W}_{t}^{-1}\\
    &= \bar{V}_{t}^{-1} + \bar{V}_{t}^{-1}Z_{t}\:\bar{W}_{t}^{-1},
\end{align*}
where the final equality follows from rearranging Equation \eqref{equation: matrices relationship}.
\end{tproof}

\potential*
\begin{tproof}
Firstly, we rewrite the norm as follows:
\begin{align*}
    \norm{X_{t}}_{M_{t - 1}} &= \norm{X_{t}}_{\bar{V}_{t - 1}^{-1} Z_{t - 1} \bar{W}_{t - 1}^{-1}}\\
    &= \sqrt{X_{t}^{T} \bar{V}_{t - 1}^{-1} Z_{t - 1} \bar{W}_{t - 1}^{-1}  X_{t}}\\
    &= \sqrt{\Tr \left(X_{t}^{T}\bar{V}_{t - 1}^{-1} Z_{t - 1} \bar{W}_{t - 1}^{-1} X_{t}\right)} \\
    &= \sqrt{\Tr \left(\bar{V}_{t - 1}^{-1}Z_{t - 1} \bar{W}_{t - 1}^{-1} X_{t}X_{t}^{T}\right)}
\end{align*}

Let $A = \bar{V}_{t - 1}^{-1} Z_{t - 1}$ and $B = \bar{W}_{t - 1} X_{t}X_{t}^{T}$. Then Lemma \ref{lemma: ab has positive eigenvalues} guarantees that $A$ and $B$ have non-negative eigenvalues, meaning:
$$\Tr(AB) = \Tr(AB^{1/2}B^{1/2}) = \Tr(B^{1/2} A B^{1/2}) \leq \Tr(B^{1/2} (\Tr(A))I B^{1/2}) = \Tr(A) \Tr(B)$$
Therefore
\begin{align*}
    \norm{X_{t}}_{M_{t - 1}} &\leq \sqrt{\Tr\left(\bar{W}_{t - 1}^{-1} X_{t}X_{t}^{T}\right) \Tr\left(\bar{V}_{t - 1}^{-1}Z_{t - 1}\right)}\\
    &\leq \frac{1}{2}\Tr\left(\bar{W}_{t - 1}^{-1} X_{t}X_{t}^{T}\right) + \frac{1}{2}\Tr\left(\bar{V}_{t - 1}^{-1}Z_{t - 1}\right) \tag{AM-GM Inequality}\\
    &= \frac{1}{2}\norm{X_{t}}_{\bar{W}_{t - 1}^{-1}}^{2} + \frac{1}{2}\sum_{s - 1}^{t - 1}\mathds{1}\left\{s + \tau_{s} > t - 1\right\}\norm{X_{s}}_{\bar{V}_{t - 1}^{-1}}^{2}\tag{Equation \eqref{equation: missing design matrix}}\\
    &\leq \frac{1 + G_{t - 1}}{2}\norm{X_{t}}_{\bar{V}_{t - 1}^{-1}}^{2} + \frac{1}{2}\sum_{s - 1}^{t - 1}\mathds{1}\left\{s + \tau_{s} > t - 1\right\}\norm{X_{s}}_{\bar{V}_{t - 1}^{-1}}^{2}\tag{Lemma \ref{lemma: matrix inverse relationship} \& \ref{lemma: product matrix bound} and $\lambda \geq 1$}\\
    &\leq \frac{1 + G_{*}}{2}\norm{X_{t}}_{V_{t - 1}^{-1}}^{2} + \frac{1}{2}\sum_{s = 1}^{t - 1}\mathds{1}\left\{s + \tau_{s} > t - 1\right\}\norm{X_{s}}_{V_{t - 1}^{-1}}^{2}\tag{$G_{t} \leq G_{*}$}
\end{align*}

Now, we are ready to reintroduce the outer summation. Doing so gives: 
\begin{align*}
      \sum_{t = 1}^{T} \norm{X_{t}}_{M_{t - 1}} &\leq \frac{1 + G_{*}}{2} \sum_{t = 1}^{T}\norm{X_{t}}_{V_{t - 1}^{-1}}^{2} + \frac{1}{2} \sum_{t = 1}^{T}\sum_{s = 1}^{t - 1}\mathds{1}\left\{s + \tau_{s} > t - 1\right\}\norm{X_{s}}_{V_{t - 1}^{-1}}^{2}\\
      &\leq \frac{1 + G_{*}}{2} \sum_{t = 1}^{T}\norm{X_{t}}_{V_{t - 1}^{-1}}^{2} + \frac{1}{2} \sum_{t = 1}^{T}\sum_{s = 1}^{t - 1}\mathds{1}\left\{s + \tau_{s} > t - 1\right\}\norm{X_{s}}_{V_{s - 1}^{-1}}^{2}\tag{$\bar{V}_{t} \succeq V_{s}$ for $t \geq s$}\\
      &\leq \frac{1 + G_{*}}{2} \sum_{t = 1}^{T}\norm{X_{t}}_{V_{t - 1}^{-1}}^{2} + \frac{1}{2} \sum_{t = 1}^{T}\tau_{t}\,\norm{X_{t}}_{V_{t - 1}^{-1}}^{2}
\end{align*}
The final equality follows from expanding the two summations and realising that the indicator ensures each term contributes to the summation $\tau_{t}$ times. Simply rearranging the above terms gives the final result.
\end{tproof}

\missing*
\begin{tproof}
The proof of this claim is similar to that found in work done on multi-armed bandits \citep{Joulani2013}. However, we extend the result so it holds for continuous delay distributions. Bernstein's inequality gives the following tail bound on sums of subgaussian random variables: 
\begin{equation*}
    \mathbb{P}\left(G_{t} - \mathbb{E}\left[G_{t}\right] \geq \frac{2}{3}\log\left(\frac{1}{\delta'}\right) + 2\sqrt{\mathbb{V}\left[G_{t}\right]\log\left(\frac{1}{\delta'}\right)} \,\right) \leq \delta'
\end{equation*}
Setting $\delta' = 6\delta/\pi^2 t^2$ and taking a union bound over all possible rounds reveals that: 
\begin{equation*}
    \mathbb{P}\left(\exists\, t\in \mathbb{N}_{1}: G_{t} - \mathbb{E}\left[G_{t}\right] \geq \frac{2}{3}\log\left(\frac{1}{\delta'}\right) + 2\sqrt{\mathbb{V}\left[G_{t}\right]\log\left(\frac{1}{\delta'}\right)}\, \right) \leq \sum_{t = 1}^{\infty}\delta' = \frac{6\delta}{\pi^2}\sum_{t = 1}^{\infty} \frac{1}{t^2} = \delta
\end{equation*}
Therefore, with probability $1 - \delta$: 
\begin{equation*}
    G_{t} \leq \mathbb{E}\left[G_{t}\right] + \frac{4}{3}\log\left(\frac{2t}{3\delta}\right) + 2\sqrt{2\mathbb{V}\left[G_{t}\right]\log\left(\frac{2t}{3\delta}\right)}
\end{equation*}
for any $t\in \mathbb{N}_{1}$. All that remains is to show that expectation and variance of the number of missing feedbacks is smaller than the expected delay. By assumption, the delays are independent. Therefore, each of the indicator variables involved in the definition of $G_{t}$ are independent. Considering its expectation reveals that: 
\begin{align*}
    \mathbb{E}\left[G_{t}\right] &=  \sum_{s = 1}^{t}\mathbb{E}\left[\mathds{1}\left\{s + \tau_{s} > t\right\}\right] =  \sum_{s = 1}^{t}\mathbb{P}\left[s + \tau_{s} > t\right] = \sum_{i = 0}^{t - 1} \mathbb{P}\left[\,\tau_{t - i} > i\,\right]\\
    &\leq \sum_{i = 0}^{\infty} \mathbb{P}\left[\,\tau > i\,\right] = \sum_{i = 0}^{\infty} \sum_{j = i + 1}^{\infty} \mathbb{P}\left[\,\tau = j\,\right] = \sum_{j = 1}^{\infty} \sum_{i = 0}^{j - 1} \mathbb{P}\left[\,\tau = j\,\right]\\
    &= \sum_{j = 1}^{\infty} j\,\mathbb{P}\left[\,\tau = j\,\right] = \mathbb{E}\left[\tau\right]
\end{align*}
for discrete delay distributions. For continuous delay distributions, we can obtain a similar result by utilising the fact that the complement of the cumulative distribution function is non-increasing: 
\begin{align*}
    \mathbb{E}\left[G_{t}\right] &=  \sum_{s = 1}^{t}\mathbb{E}\left[\mathds{1}\left\{s + \tau_{s} > t\right\}\right] = \sum_{s = 1}^{t}\mathbb{P}\left[\tau_{s} > t - s\right] = \sum_{x = 0}^{t - 1}\mathbb{P}\left[\tau > x\right]\\
    &\leq 1 + \int_{0}^{t} \mathbb{P}\left[\,\tau > x\,\right] dx = 1 + \int_{0}^{t} \int_{x}^{\infty} f_{\tau}(y)\,dy \,dx \tag{Setting $x = t - s$}\\
    &\leq 1 + \int_{0}^{\infty} \int_{x}^{\infty} f_{\tau}(x)\,dy \,dx = 1 + \int_{0}^{\infty} \int_{0}^{y} f_{\tau}(y)\,dx\,dy\tag{Tonelli's Theorem}\\
    &= 1 + \int_{0}^{\infty} \left[x f_{\tau}(y)\right]_{0}^{y} dy = 1 + \int_{0}^{\infty} y f_{\tau}(y) dy\\
    &= 1 + \mathbb{E}\left[\tau\right]
\end{align*}

Similarly, looking at the variance reveals that: 
\begin{align*}
    \mathbb{V}\left[G_{t}\right] &= \sum_{s = 1}^{t}\mathbb{V}\left[\mathds{1}\left\{s + \tau_{s} \geq t\right\}\right] \leq \sum_{s = 1}^{t} \mathbb{E}\left[\mathds{1}\left\{s + \tau_{s} \geq t\right\}^{2}\right] = \mathbb{E}\left[G_{t}\right]\,,
\end{align*}
which is smaller than the expected delay. Therefore, 
\begin{equation*}
    G_{t} \leq 1 + \mathbb{E}\left[\tau\right] + \frac{4}{3}\log\left(\frac{2t}{3\delta}\right) + 2\sqrt{2\mathbb{E}\left[G_{t}\right]\log\left(\frac{2t}{3\delta}\right)}
\end{equation*}
as required. 
\end{tproof}

\subsection{Supporting Lemmas}
Proving Lemma \ref{lemma: product elliptical potential} requires Lemma \ref{lemma: product matrix bound}, which itself requires two additional results. We state and prove all three of these results in this subsection.

\begin{lemma}\label{lemma: ab has positive eigenvalues}
Let $A\in\mathbb{R}^{d\times d}$ and $B\in\mathbb{R}^{d\times d}$ be two symmetric positive semi-definite matrices. Then, $A^{1/2}BA^{1/2}$ and $AB$ share the same set of eigenvalues. Further, these eigenvalues are all non-negative.
\end{lemma}
\begin{proof}
Since $A$ is positive semi-definite, we can utilise the spectral decomposition to show that: $AB = A^{1/2} A^{1/2} B$. Suppose $AB$ has an eigenvalue equal to $\lambda$. Then, there exists a non-zero eigenvector such that: 
\begin{equation*}
    AB \vec{v} = A^{1/2} A^{1/2} B \vec{v} = \lambda \vec{v}
\end{equation*}
Pre-multiplying both sides of the above equation by the same matrix gives: 
\begin{equation*}
    A^{1/2} B AB \vec{v} = A^{1/2} B A^{1/2} \left(A^{1/2}B\right) \vec{v} = \lambda \left(A^{1/2} B\right) \vec{v}
\end{equation*}
Thus, $AB$ and $A^{1/2}BA^{1/2}$ share the same set of eigenvalues, albeit with different eigenvectors, verifying the first statement of the lemma. Now, $A^{1/2} B A^{1/2}$ is symmetric, because:
\begin{equation*}
    \left(A^{1/2} B A^{1/2}\right)^{T} = (A^{1/2})^{T} B^{T} (A^{1/2})^T = A^{1/2}BA^{1/2}
\end{equation*}
Further, it is positive semi-definite, because:
\begin{equation*}
    \underbrace{x^T A^{1/2}}_{\tilde{x}^T} B \underbrace{A^{1/2} x}_{\tilde{x}} = \tilde{x}^{T} B \tilde{x} \geq 0
\end{equation*}
The final inequality follows from the fact that $B$ is positive semi-definite. Therefore, $A^{1/2} B A^{1/2}$ must have non-negative eigenvalues, as it is symmetric and positive semi-definite. Recall $AB$ and $A^{1/2} B A^{1/2}$ shares the same set of eigenvalues. Therefore, $AB$ has non-negative eigenvalues too. 
\end{proof}

\begin{lemma}\label{lemma: missing matrix bound}
Let $Z_{t}$ and $G_{t}$ be the missing design matrix and the number of missing feedbacks at the end of the $t$-th round, respectively. Then, $\lambda_{1}(Z_{t}) \leq G_{t}$.
\end{lemma}
\begin{proof}
By Equations \eqref{equation: total design matrix}, \eqref{equation: observed design matrix} and \eqref{equation: missing design matrix}, we have that:
\begin{align*}
Z_{t} = \sum_{s\leq t} \mathds{1}\{s + \tau_{s} > t\} X_{s}X_{s}^{T},
\end{align*}
Clearly, $X_{s}X_{s}^{T}$ is a symmetric matrix, as it is the outer product of two vectors. The Courant–Fischer–Weyl min-max principle shows that: 
\begin{align*}
    \lambda_{1}\left(X_{s}X_{s}^{T}\right) \leq \norm{X_{s}}_{2}^{2} \leq 1,
\end{align*}
where the final inequality follows from assuming that the vectors are appropriately normalised. Applying Weyl's inequality repeatedly to each symmetric matrix in the summation and utilising the above result gives:
\begin{align*}
\lambda_{1}\left(Z_{t}\right) \leq  \sum_{s\leq t} \mathds{1}\{s + \tau_{s} > t\} \lambda_{1}\left(I\right) = G_{t}
\end{align*}
as required.
\end{proof}

\begin{lemma}\label{lemma: product matrix bound}
Let $\bar{V}_{t}$, $\bar{W}_{t}$ and $Z_{t}$ be the total, observed and missing gram matrices, respectively. Then, 
\begin{equation*}
    \frac{G_{t}}{\lambda}\,\bar{V}_{t}^{-1} \succeq \bar{V}_{t}^{-1} Z_{t}\: \bar{W}_{t}^{-1} = M_{t}
\end{equation*}
\end{lemma}
\begin{tproof}
Firstly, $A \succeq B \iff A - B \succeq 0$. Therefore, we focus on proving that the difference between the two matrices is positive semi-definite. That is, we prove that: 
\begin{equation}\label{equation: difference matrix}
    D_{t} = \frac{G_{t}}{\lambda}\,\bar{V}_{t}^{-1} - \bar{V}_{t}^{-1} Z_{t}\:\bar{W}_{t}^{-1}
\end{equation}
is positive semi-definite. To do so, we take a four-stepped approach. Below is an overview of these four steps. 
\begin{itemize}
    \item [1.] Firstly, we show that the matrix of \eqref{equation: difference matrix} is symmetric. 
    \item [2.] Next, we define a similar matrix and prove that it has the same set of eigenvalues as that of \eqref{equation: difference matrix}.
    \item [3.] Then, we show that all the eigenvalues of the similar matrix are non-negative.
    \item [4.] Finally, we chain the above three steps in reverse order and recall basic facts about symmetric matrices to prove the claim.
\end{itemize}

\textit{Step 1.} Indeed, $D_{t}$ is the difference of two symmetric matrices. From Equations \eqref{equation: total design matrix}, \eqref{equation: observed design matrix} and \eqref{equation: missing design matrix}, the first matrix is symmetric, as it is just the total gram matrix scaled by a constant. Also, the second matrix is symmetric because it is the difference between the two symmetric matrices:  
\begin{equation*}
\bar{V}_{t}^{-1} Z_{t} \bar{W}_{t}^{-1} =  \bar{V}_{t}^{-1} \left(\bar{V}_{t} - \bar{W}_{t}\right)\bar{W}_{t}^{-1} = \bar{W}_{t}^{-1} - \bar{V}_{t}^{-1}
\end{equation*}

Therefore, $D_{t}$ is symmetric, as it is the difference between two symmetric matrices. Indeed, a symmetric matrix must have all non-negative eigenvalues for positive semi-definiteness to hold. Thus, it is sufficient to find a matrix with the same eigenvalues and show that its quadratic form is greater than or equal to zero, for which non-negative eigenvalues is a necessary condition.\\ 

\textit{Step 2.} To that end, we define the following matrix:
\begin{equation*}
\tilde{D}_{t} \coloneqq \bar{V}_{t}^{-1/2}\left(\frac{G_{t}}{\lambda}I - Z_{t}\:\bar{W}_{t}^{-1}\right)\bar{V}_{t}^{-1/2}
\end{equation*}
Applying Lemma \ref{lemma: ab has positive eigenvalues} with $A = \bar{V_{t}}$ and $B = (G_{t}/\lambda) I - Z_{t}\bar{W}_{t}^{-1}$ reveals that $D_{t}$ and $\tilde{D}_{t}$ share the same set of eigenvalues, albeit with different eigenvectors.\\

 \textit{Step 3.} Showing $\tilde{D}_{t} \succeq 0$ proves it must have non-negative eigenvalues, as this is a necessary condition for the positive semi-definiteness of an arbitrary (possibly non-symmetric) matrix. Utilising the definition of positive semi-definiteness, we can verify whether or not this holds by checking if: $x^T \tilde{D}_{t} x \geq 0$. To do so, we first decompose the matrix into the sum of symmetric and anti-symmetric matrices: 
$$
\tilde{D}_{t} = \frac{1}{2}\left(\tilde{D}_{t} + \tilde{D}_{t}^{T}\right) + \frac{1}{2} \left(\tilde{D}_{t} - \tilde{D}_{t}^{T}\right),
$$
Then, we use the fact that: 
$$y = x^{T} \left(\tilde{D}_{t} - \tilde{D}_{t}^{T}\right) x = \left(x^{T} \left(\tilde{D}_{t} - \tilde{D}_{t}^{T}\right) x\right)^T =x^{T} \left(\tilde{D}_{t} - \tilde{D}_{t}^{T}\right)^{T} x = - x^{T} \left(\tilde{D}_{t} - \tilde{D}_{t}^{T}\right) x = -y\,,$$
which holds if and only if $y = 0$. Doing so gives:
\begin{align}
    x^T \tilde{D}_{t} x &= \frac{1}{2} x^T \left(\tilde{D}_{t} + \tilde{D}_{t}^{T}\right)x + \frac{1}{2} x^T \left(\tilde{D}_{t} - \tilde{D}_{t}^{T}\right) x = \frac{1}{2} x^T \left(\tilde{D}_{t} + \tilde{D}_{t}^{T}\right)x \nonumber \\
    &= \frac{1}{2} x^T\left(\bar{V}_{t}^{-1/2}\left(\frac{2G_{t}}{\lambda}I - Z_{t}\:\bar{W}_{t}^{-1} - \bar{W}_{t}^{-1} Z_{t}\right)\bar{V}_{t}^{-1/2}\right)x \nonumber\\
    &= \frac{G_{t}}{\lambda} x^T \bar{V}_{t}^{-1} x - \frac{1}{2}x^T\bar{V}_{t}^{-1/2}\left(Z_{t}\:\bar{W}_{t}^{-1} + \bar{W}_{t}^{-1} Z_{t}\right)\bar{V}_{t}^{-1/2}x  \nonumber\\
    &= \frac{G_{t}}{\lambda}\norm{x}_{\bar{V}_{t}^{-1}}^{2} - \frac{1}{2}x^T \bar{V}_{t}^{-1/2}\left(Z_{t}\:\bar{W}_{t}^{-1} + \bar{W}_{t}^{-1} Z_{t}\right)\bar{V}_{t}^{-1/2}x \nonumber
\end{align}
Now, $A = Z_{t}\:\bar{W}_{t}^{-1} + \bar{W}_{t}^{-1} Z_{t}$ is a real-value symmetric matrix. Therefore, its eigendecomposition is given by $A = Q\Lambda Q^T$ where $\Lambda$ is a diagonal matrix containing its eigenvalues and $Q$ is an orthogonal matrix whose columns contain its unit eigenvectors. 
\begin{align}
    x^T \tilde{D}_{t} x &= \frac{G_{t}}{\lambda}\norm{x}_{\bar{V}_{t}^{-1}}^{2}  - \frac{1}{2}x^T \bar{V}_{t}^{-1/2}Q\Lambda Q^T \bar{V}_{t}^{-1/2}x  \nonumber\\
    &= \frac{G_{t}}{\lambda}\norm{x}_{\bar{V}_{t}^{-1}}^{2}  - \frac{1}{2}\tilde{x}^T \Lambda \tilde{x}  \nonumber \tag{Setting $\tilde{x} = Q^T \bar{V}_{t}^{-1/2}x$}\\
    &= \frac{G_{t}}{\lambda}\norm{x}_{\bar{V}_{t}^{-1}}^{2}  - \frac{1}{2} \sum_{i = 1}^{d}\lambda_{i}\,\tilde{x}_{i}^{2} \label{eqn: mtilde norm}
\end{align}
where $\lambda_{i}$ is the $i$-th largest eigenvalue of matrix $A = Z_{t}\bar{W}_{t}^{-1} + \bar{W}_{t}^{-1} Z_{t}$. Indeed, $A$ is a symmetric matrix positive definite matrix. It is symmetric because it is the sum of a matrix and its transpose and it is positive semi-definite because: 
$$
\norm{x}_{A}^{2} = 2\underbrace{x^T Z_{t}}_{\in\mathbb{R}^{d}}(\bar{W}_{t}^{-1}) \underbrace{Z_{t} x}_{\in\mathbb{R}^{d}}
$$
and $\bar{W}_{t}^{-1}\succ 0$. Thus, it follows that the singular values $A$ are the absolute values of its eigenvalues. Define $\sigma_{i}$ as the $i$-th largest singular value of matrix $Z_{t}W_{t}\left(\lambda_{*}\right)^{-1} + W_{t}\left(\lambda_{*}\right)^{-1} Z_{t}$. Then, we have that:
\begin{align*}
        \eqref{eqn: mtilde norm}  &\geq  \frac{G_{t}}{\lambda} \norm{x}_{V_{t}\left(\lambda_{*}\right)^{-1}}^{2} - \frac{1}{2} \sum_{i = 1}^{d}\sigma_{i}\,\tilde{x}_{i}^{2} \tag{Since $x \leq \abs{x}$}\\
        &\geq \frac{G_{t}}{\lambda} \norm{x}_{\bar{V}_{t}^{-1}}^{2} - \frac{1}{2} \sigma_{\max}\norm{x}_{\bar{V}_{t}^{-1}}^{2} \tag{Definition of $\tilde{x}$}\\
        &= \frac{G_{t}}{\lambda} \norm{x}_{\bar{V}_{t}^{-1}}^{2} - \frac{1}{2} \norm*{Z_{t}\bar{W}_{t}^{-1} + \bar{W}_{t}^{-1} Z_{t}}_{2}\,\norm{x}_{\bar{V}_{t}^{-1}}^{2} \tag{For $A\in \mathbb{R}^{d\times d}$: $\sigma_{\max} = \norm{A}_{2}$}\\
        &\geq \frac{G_{t}}{\lambda} \norm{x}_{\bar{V}_{t}^{-1}}^{2} - \frac{1}{2} \left(\norm{Z_{t}\bar{W}_{t}^{-1}}_{2} + \norm{\bar{W}_{t}^{-1} Z_{t}}_{2}\right)\norm{x}_{\bar{V}_{t}^{-1}}^{2} \tag{Matrix Norm is Sub-additive}\\
        &\geq \frac{G_{t}}{\lambda} \norm{x}_{\bar{V}_{t}^{-1}}^{2} - \norm{Z_{t}}_{2}\norm{\bar{W}_{t}^{-1}}_{2}\,\norm{x}_{\bar{V}_{t}^{-1}}^{2}\tag{Matrix Norm is Sub-multiplicative}
\end{align*}

By definition, $\norm{\bar{W}_{t}^{-1}}_{2} \leq 1/\lambda$. Further, Lemma \ref{lemma: missing matrix bound} tells us that $\norm{Z_{t}}_{2} \leq G_{t}$. 
Substituting this into the above gives: 
\begin{align*}
   x^T \tilde{D}_{t} x &\geq \frac{G_{t}}{\lambda} \norm{x}_{\bar{V}_{t}^{-1}}^{2} - \norm{Z_{t}}_{2}\norm{\bar{W}_{t}^{-1}}_{2}\,\norm{x}_{\bar{V}_{t}^{-1}}^{2}\\
   &\geq \frac{G_{t}}{\lambda} \norm{x}_{\bar{V}_{t}^{-1}}^{2} - \frac{G_{t}}{\lambda_{*}} \norm{x}_{\bar{V}_{t}^{-1}}^{2}\\
   &= 0
\end{align*}

\textit{Step 4.} Now, Step 3 shows $\tilde{D}_{t}$ is positive semi-definite, implying all of its eigenvalues are non-negative. Therefore, Step 2 shows $D_{t}$ has non-negative eigenvalues. Finally, Step 1 shows $D_{t}$ is a symmetric matrix. Since $D_{t}$ is symmetric, non-negative eigenvalues implies positive semi-definiteness. Finally, $D_{t} \succeq 0$ implies that: 
$$
\frac{G_{t}}{\lambda}\bar{V}_{t}^{-1} \succeq  \bar{V}_{t}^{-1} Z_{t}\:\bar{W}_{t} = M_{t}\,,
$$
as required. 
\end{tproof}
\section{Standard Results}\label{section: standard results}
Here, we present a selection of well-known tail bounds for subgaussian and subexponential random variables that find use in our paper. Additionally, we provide proof of the elliptical potential lemma. 

\begin{lemma}[Bernstein's Inequality]
Let $\{X_{t}\}_{t = 1}^{n}$ be a sequence of independent and identically distributed $\sigma_{t}$-subgaussian random variables. Define $S_{n} = X_{1} + X_{2} + \cdots + X_{n}$. Then, 
\begin{equation*}
    \mathbb{P}\left(S_{n} - \mathbb{E}\left[S_{n}\right] \geq \frac{2}{3}\log\left(\frac{1}{\delta}\right) + 2\sqrt{\mathbb{V}\left[S_{n}\right]\log\left(\frac{1}{\delta}\right)}\right)\leq \delta
\end{equation*}
\end{lemma}

\begin{lemma}[Subexponential Tail Bounds]
Suppose $\{\tau_{t}\}_{t = 1}^{\infty}$ are $(v, \alpha)$ subexponential random variables. Then, 
\begin{equation*}
    \mathbb{P}\left(\tau_{t} - \mathbb{E}\left[\tau_{t}\right] \geq \epsilon \right) \leq \exp\left(-\frac{1}{2} \min\left\{\frac{\epsilon^2}{v^2}, \frac{\epsilon}{\alpha}\right\}\right)
\end{equation*}
Therefore, with probability $1 - \delta$:
\begin{equation*}
\tau_{t} \leq \mathbb{E}\left[\tau_{t}\right] + \min\left\{\sqrt{2 v^2 \log\left(2t/3\delta\right)}, 2\alpha \log\left(2t/3\delta'\right)\right\}
\end{equation*}
for any $t \in \mathbb{N}_{1}$.
\end{lemma}

\begin{lemma}[Elliptical Potential Lemma]\label{lemma: elliptical potential}
Let $\{X_{t}\}_{t = 1}^{\infty}$ be an arbitrary sequence of $d$-dimensional vectors such that $\norm{X_{t}}_{2}^{2} \leq 1$. Define $V_{0} = \lambda I$, $V_{t} = \sum_{s = 1}^{t} X_{s}X_{s}^T$ and $\bar{V}_{t} = V_{0} + V_{t}$. Then,
$$
\sum_{t = 1}^{T}\norm{X_{t}}_{\bar{V}_{t - 1}^{-1}}^{2} \leq 2\log\left(\frac{\det\left(\bar{V}_{T}\right)}{\det\left(\bar{V}_{0}\right)}\right)\leq  2d\log\left(\frac{d\lambda + T}{d\lambda}\right)
$$
for $\lambda \geq 1/2$.
\end{lemma}
\begin{proof}
For completeness, we provided a detailed proof of the elliptical potential lemma using the arguments of \citet{Abbasi-Yadkori2011}. However, we note that this is not the only way to obtain the stated result. \citet{Carpentier2020} prove the lemma using insights from linear algebra.

Firstly, notice that: 
\begin{equation*}
    \norm{x}_{\bar{V}_{t}}^{2} 
    = \lambda \norm{x}_{2}^{2} + \sum_{s = 1}^{t}\left(x^T X_{s}\right)\left(X_{s}^{T} x\right) = \lambda \norm{x}_{2}^{2} + \sum_{s = 1}^{t} \norm{X_{s}^{T} x}_{2}^{2} \geq \lambda \norm{x}_{2}^{2} > 0 \implies \norm{x}_{\bar{V}_{t}^{-1}}^{2} > 0
\end{equation*}
Additionally, $\lambda \geq 1/2$ and $\norm{X_{t}}_{2}^{2} \leq 1$. Therefore,
$$
\norm{x}_{\bar{V}_{t}^{-1}}^{2} \leq \norm{x}_{\bar{V}_{0}^{-1}}^{2} = \frac{1}{\lambda}\norm{x}_{2}^{2} \leq \frac{1}{\lambda}
$$
Consequently, 
$$0 < \norm{X_{t}}_{\bar{V}_{t}^{-1}}^{2} \leq \frac{1}{\lambda} \leq 2$$

Since $x < 2\ln(1 + x)$ for any $0 < x \leq 2$, we have that:
\begin{equation}\label{equation: ep start}
    \sum_{t = 1}^{T} \norm{X_{t}}_{\bar{V}_{t}^{-1}}^{2} \leq 2\sum_{t = 1}^{T}\log\left(1 + \norm{X_{t}}_{\bar{V}_{t}^{-1}}^{2}\right) = 2\log\left(\,\prod_{t = 1}^{T}\left(1 + \norm{X_{t}}_{\bar{V}_{t}^{-1}}^{2}\right)\right)
\end{equation}

Now, proving the first inequality amounts to relating the term inside the logarithm to the determinants of the matrices. By Definition, we have that: 
$$
\bar{V}_{t} = \bar{V}_{t - 1} + X_{t}X_{t}^{T} = \bar{V}_{t - 1}^{1/2}\left(I + \bar{V}_{t - 1}^{-1/2} X_{t}X_{t}^{T}\bar{V}_{t - 1}^{-1/2}\right)\bar{V}_{t - 1}^{1/2}
$$
and
\begin{align*}
    \det\left(\bar{V}_{n}\right) &= \det \left(\bar{V}_{n - 1}^{1/2}\left(I + \bar{V}_{n - 1}^{-1/2} X_{n}X_{n}^{T}\bar{V}_{n - 1}^{-1/2}\right)\bar{V}_{n - 1}^{1/2}\right)\\
    &= \det\left(\bar{V}_{n - 1}\right)\det\left(I + \bar{V}_{n - 1}^{-1/2} X_{n}X_{n}^{T}\bar{V}_{n - 1}^{-1/2}\right)\tag{Properties of Determinants}\\
    &= \det\left(\bar{V}_{n - 1}\right)\left(1 + X_{n}^{T}\bar{V}_{n - 1}^{-1} X_{n}\right)\tag{Matrix Determinant Lemma}\\
    &= \det\left(\bar{V}_{n - 1}\right)\left(1 + \norm{X_{n}}_{\bar{V}_{n - 1}^{-1}}^{2}\right)\tag{By Positive Definiteness}\\
    &= \det\left(V_{0}\right)\prod_{t = 1}^{n}\left(1 + \norm{X_{t}}_{\bar{V}_{t - 1}^{-1}}^{2}\right)
\end{align*}
Rearranging and plugging this into \eqref{equation: ep start} gives: 
\begin{equation*}
    \sum_{t = 1}^{T} \norm{X_{t}}_{\bar{V}_{t}^{-1}}^{2} \leq 2\log\left(\,\prod_{t = 1}^{T}\left(1 + \norm{X_{t}}_{\bar{V}_{t}^{-1}}^{2}\right)\right) = 2\log\left(\frac{\det\left(V_{T}\right)}{\det\left(V_{0}\right)}\right) 
\end{equation*}
proving the first inequality. Lemma \ref{lemma: trace-determinant} proves the second inequality, completing the proof. 
\end{proof}
 




\begin{lemma}\label{lemma: trace-determinant}
Let $\{X_{t}\}_{t = 1}^{\infty}$ be an arbitrary sequence of $d$-dimensional vectors such that $\norm{X_{t}}_{2}^{2} \leq 1$. Define $V_{0} = \lambda I$, $V_{t} = \sum_{s = 1}^{t} X_{s}X_{s}^T$ and $\bar{V}_{t} = V_{0} + V_{t}$. Then,
$$
2\log\left(\frac{\det\left(\bar{V}_{T}\right)}{\det\left(\bar{V}_{0}\right)}\right)\leq  2d\log\left(\frac{d\lambda + T}{d\lambda}\right)
$$
for $\lambda \geq 1/2$.
\end{lemma}
\begin{proof}
Let $\bar{V}_{T}$ have eigenvalues $\lambda_{1} \geq \lambda_{2} \geq \cdots \geq \lambda_{d}$. Then, 
\begin{align*}
    2\log\left(\frac{\det\left(V_{T}\right)}{\det\left(V_{0}\right)}\right)  &= 2\log\left(\frac{\prod_{i = 1}^{d}\lambda_{i}}{\lambda^{d}}\right)\\
    &\leq 2\log\left(\frac{\left(\frac{1}{d}\sum_{i = 1}^{d}\lambda_{i}\right)^d}{\lambda^{d}}\right) = 2\log\left(\left(\frac{\sum_{i = 1}^{d}\lambda_{i}}{d \lambda}\right)^{d}\right)\tag{AM-GM Inequality}\\
    &= 2\log\left(\left(\frac{\Tr\left(\bar{V}_{T}\right)}{d \lambda}\right)^{d}\right) = 2\log\left(\left(\frac{\Tr\left(V_{0} + \sum_{t = 1}^{T}X_{t}X_{t}^T\right)}{d \lambda}\right)^{d}\right)\\
    &= 2\log\left(\left(\frac{\Tr\left(V_{0}\right) + \sum_{t = 1}^{T}\Tr\left(X_{t}X_{t}^T\right)}{d \lambda}\right)^{d}\right) = 2\log\left(\left(\frac{d\lambda + \sum_{t = 1}^{T}\Tr\left(X_{t}^T X_{t}\right)}{d \lambda}\right)^{d}\right)\\
    &= 2\log\left(\left(\frac{d\lambda  + \sum_{t = 1}^{T}\norm{X_{t}}_{2}^{2}}{d \lambda}\right)^{d}\right)\\
    &\leq 2\log\left(\left(\frac{d\lambda  + T}{d \lambda}\right)^{d}\right)\\
    &= 2d\log\left(\frac{d\lambda  + T}{d \lambda}\right)
\end{align*}
which completes the proof.
\end{proof}

\section{Additional Experimental Results}\label{sec: additional experiments}
Here, we present additional experimental results for the linear and logistic bandits under delayed feedback with delays drawn from the uniform and Pareto distributions. 

\subsection{Uniform Delays}
\begin{figure*}[h!]
    \centering
    \includegraphics[width = \textwidth]{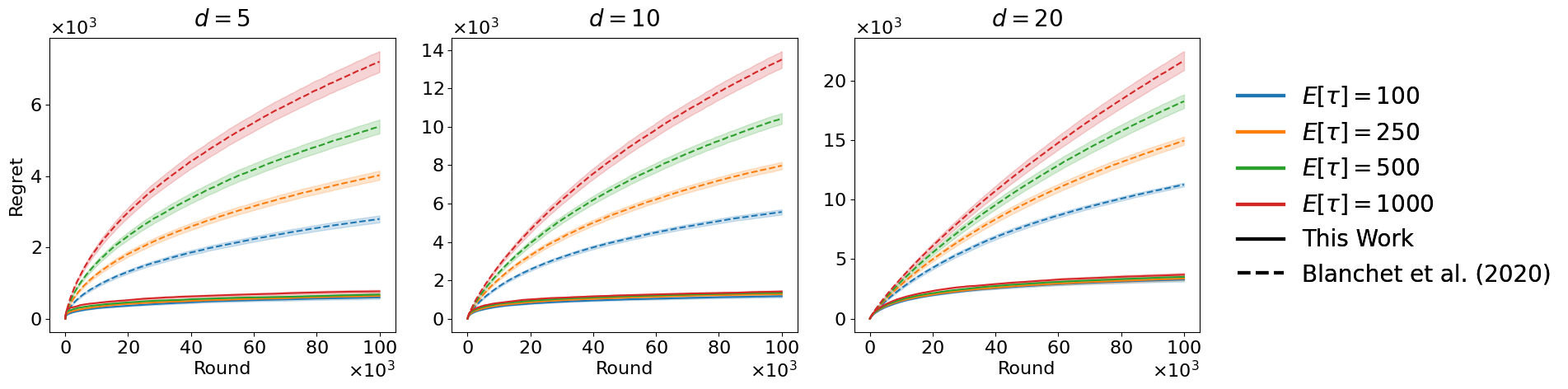}
    \caption{Linear Bandit \& Uniform Delays.}
\end{figure*}

\begin{figure*}[h!]
    \centering
    \includegraphics[width = \textwidth]{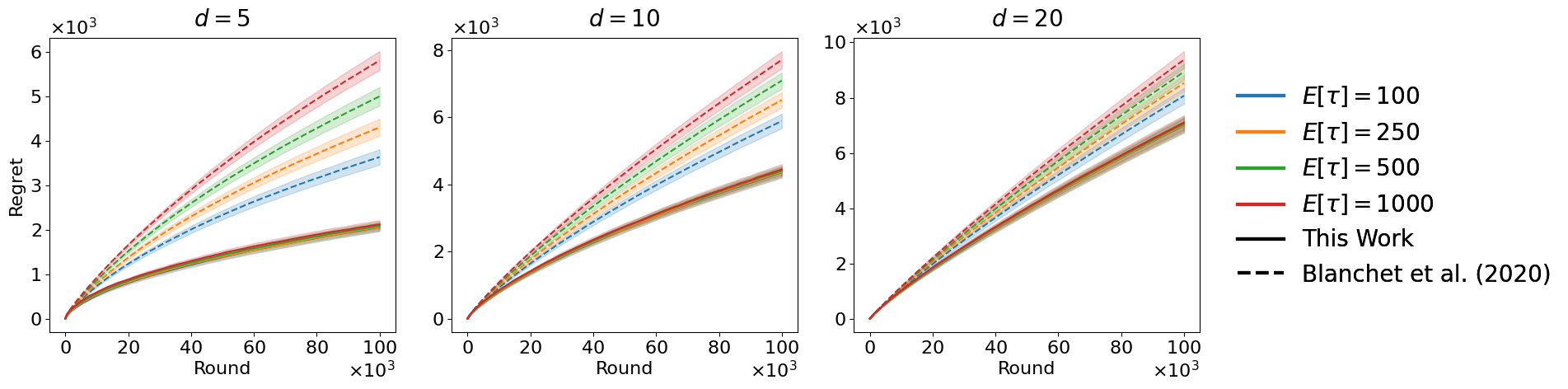}
    \caption{Logistic Bandit \& Uniform Delays.}
\end{figure*}

\newpage
\subsection{Pareto Delays}
\begin{figure*}[h!]
    \centering
    \includegraphics[width = \textwidth]{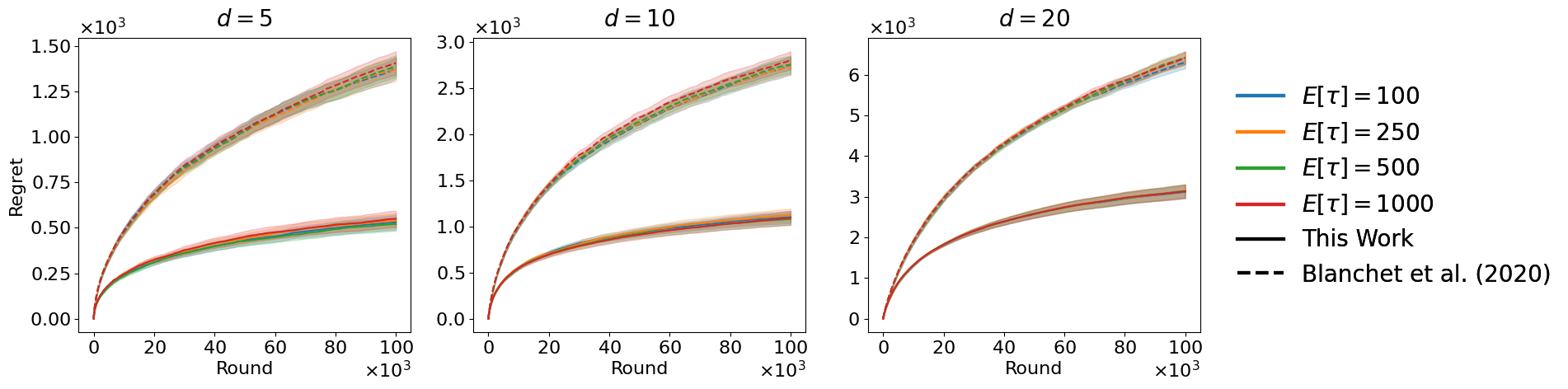}
    \caption{Linear Bandit \& Pareto Delays.}
    \label{figure: linear bandit pareto}
\end{figure*}

\begin{figure*}[ht!]
    \centering
    \includegraphics[width = \textwidth]{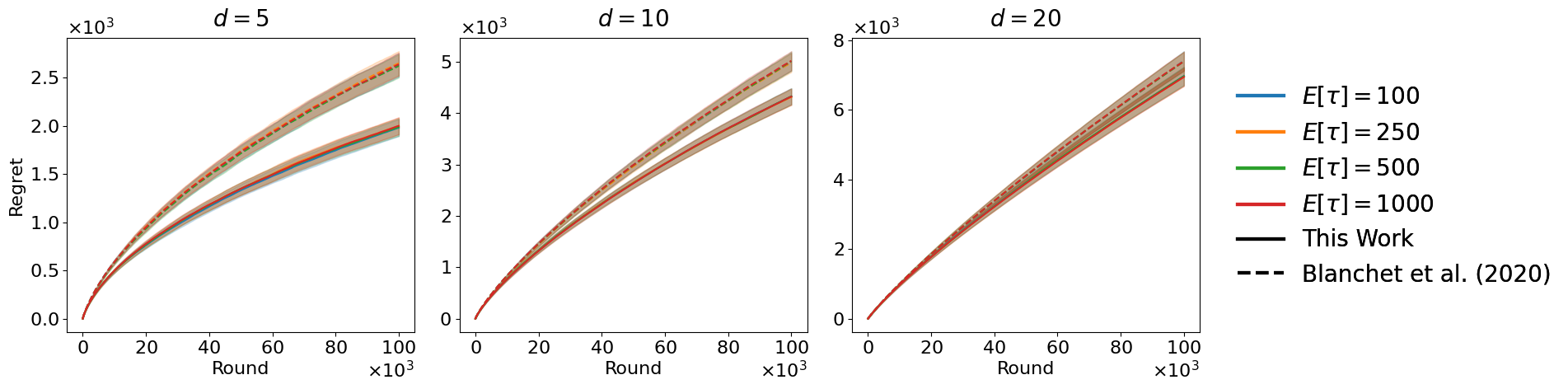}
    \caption{Logistic Bandit \& Pareto Delays}
    \label{figure: logistic bandit pareto}
\end{figure*}
\end{appendix}

\end{document}